\documentclass[twoside]{article}

\usepackage{microtype}
\usepackage{graphicx}
\usepackage{subfigure}
\usepackage{booktabs} 

\newcommand\numberthis{\addtocounter{equation}{1}\tag{\theequation}}

\def\code#1{\texttt{#1}}

\usepackage[accepted]{aistats2021}
\usepackage{dblfloatfix}

\usepackage{tikz}
\usetikzlibrary{automata,decorations.markings,positioning,arrows}
\usepackage{wrapfig, amsmath, amsthm, mathtools}

\usepackage[utf8]{inputenc} 
\usepackage[T1]{fontenc}    
\usepackage{xcolor}
\usepackage[colorlinks,citecolor=blue,urlcolor=blue,linkcolor=blue,linktocpage=true]{hyperref}

\usepackage{url, float}            
\usepackage{amsfonts}       
\usepackage{amssymb}
\usepackage{nicefrac}
\usepackage[round]{natbib}

\usepackage{xcolor}
\newcommand{\red}{\color{black}}
\usepackage{thm-restate}

\newcommand{\indep}{\rotatebox[origin=c]{90}{$\models$}}
\newcommand{\ATE}{\mathrm{ATE}}
\newcommand{\Do}{\mathrm{do}}
\usepackage[colorlinks]{hyperref}    
\newtheorem{theorem}{Theorem}

\newtheorem{lemma}{Lemma}

\begin{document}

\twocolumn[
\aistatstitle{Causal Inference with Selectively Deconfounded Data}
\aistatsauthor{Kyra Gan \And Andrew A. Li \And Zachary C. Lipton \And Sridhar Tayur}
\aistatsaddress{Carnegie Mellon University, Pittsburgh, PA 15213\\
\texttt{\{\href{mailto:kyragan@cmu.edu}{kyragan},\href{mailto:aali1@cmu.edu}{aali1},\href{mailto:zlipton@cmu.edu}{zlipton},\href{mailto:stayur@cmu.edu}{stayur}\}@cmu.edu}}
]



\begin{abstract}
Given only data generated by a standard confounding graph with unobserved confounder, 
the Average Treatment Effect (ATE) is not identifiable.
To estimate the ATE, a practitioner must then either 
(a) collect deconfounded data;
(b) run a clinical trial; or 
(c) elucidate further properties of the causal graph 
that might render the ATE identifiable.
In this paper, we consider the benefit of incorporating 
a large \emph{confounded} observational dataset (\emph{confounder unobserved})
alongside a small \emph{deconfounded} observational dataset (\emph{confounder revealed})
when estimating the ATE.
Our theoretical results {\red suggest} that the inclusion of confounded data 
can significantly reduce the quantity of deconfounded data required
to estimate the ATE to within a desired accuracy level.
Moreover, in some cases---say, genetics---we could imagine 
retrospectively selecting samples to deconfound.
We demonstrate that by 
actively selecting these samples 
based upon the (already observed) treatment and outcome,
we can reduce sample complexity further.
Our theoretical and empirical results establish
that the worst-case relative performance of our approach 
(vs. a natural benchmark) is bounded
while our best-case gains are unbounded. 
Finally, we demonstrate the benefits of selective deconfounding
using a large real-world dataset related to genetic mutation in cancer.


\end{abstract}

\section{Introduction}
\label{sec:intro}



The fundamental problem in causal inference is to estimate causal effects using \emph{observational} data. 
This task is particularly motivated by scenarios
when experiments are infeasible.
While the 
literature
typically addresses a rigid setting in which confounders
are either \emph{always} or \emph{never} observed,
in many applications we might 
observe confounders for a {\em subset} of samples. 
For example, in healthcare, 
a particular gene might be suspected 
to confound the relation between a behavior 
and a health outcome of interest.
Due to the high cost of genetic tests,
we might only be able to afford 
to reveal the value of the genetic confounder
for a subset of patients.
Note that for a variable such as a genetic mutation,
we might observe retrospectively, even after the treatment
and outcome have been observed. 
We call this process of revealing the value 
of an (initially unobserved) confounder 
\emph{deconfounding}, and the samples where treatment, outcome, and confounders are all observed \emph{deconfounded data}.



So motivated, this paper addresses
the middle ground 
along the confounded-deconfounded spectrum.
Naively, one could estimate the 
ATE
with standard methods using only the deconfounded data. 
%
%
%
First, we ask: 
{\em how much can we improve our ATE estimates
by incorporating confounded data over approaches
that rely on deconfounded data alone?}
Second, motivated by the setting in which
our confounders are genetic traits
that might be retrospectively observed
for cases with known treatments and outcomes,
we introduce the problem of 
\emph{selective deconfounding}---allocating 
a fixed budget for revealing the confounder
based upon observed treatments and outcomes.
This prompts our second question: 
{\em what is the optimal policy for selecting data to deconfound?}
To our knowledge, this is the first paper that focuses
on the case where ample (cheaply-acquired) confounded data 
is available and we can select only few confounded samples to deconfound (expensive).

%
%
%


We address these questions for a standard confounding graph
where the treatment and outcome are binary, and the confounder is categorical.
First, we propose a simple method for incorporating confounded data
that achieves a constant-factor improvement 
in 
ATE estimation error.
In short, the inclusion of (infinite) confounded data
reduces the number of free parameters to be estimated, 
improving our estimates of the remaining parameters.
Moreover, due to the multiplicative factors in the causal functional,
errors in parameter estimates can compound.
Thus, our improvements in parameter estimates
yield greater benefits in estimating treatment effects. 
For binary confounders, our numerical results show that on average, 
over problem instances selected uniformly on the parameter simplex,
our method achieves roughly $2.5\times$ improvements in ATE estimation error.

Next, we show that we can 
reduce 
error further
by actively choosing which samples to deconfound. 
Our proposed policy for selecting samples dominates reasonable benchmarks.
In the worst case, our method requires 
no more than $2\times$ as many samples 
as a natural sampling policy 
and our best-case gains are unbounded. 
Moreover, our qualitative analysis characterizes those situations 
most favorable/unfavorable for our method.
We extend our work to the scenario where 
only a finite amount of confounded samples is present,
demonstrating
our qualitative insights continue to apply {\red (Appendix~\ref{app:finite})}. 
Additionally, we 
validate our methods using COSMIC~\citep{tate2019cosmic, cosmic2019}, 
a real-world dataset containing cancer types,
genetic mutations, and other patient features, 
showing  the practical benefits 
of our proposed sampling policy.
Throughout the paper, we implicitly assume that the \emph{confounded} data was sampled i.i.d. from the target population of interest
(but our policy for selecting data to deconfound need not be).


\begin{figure}[!t]
\centering
\includegraphics[width=0.35\linewidth]{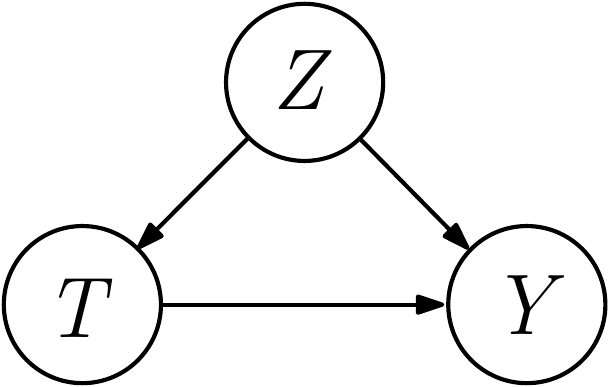}

\caption{Causal graph with treatment $T$, outcome $Y$, and selectively observed confounder $Z$}
\label{fig:causal_diagram}
\end{figure}

\section{Related Work}
\label{sec:lit}
Causal inference has been studied thoroughly under the ignorability assumption, i.e.,
no unobserved confounding
\citep{neyman1923applications, rubin1974estimating, holland1986statistics}.
Some approaches for estimating the ATE under ignorability
include inverse propensity score weighting
\citep{rosenbaum1983central, hirano2003efficient, mccaffrey2004propensity},
matching \citep{dehejia2002propensity}, the backdoor adjustment \citep{pearl1995causal}, and targeted learning~\citep{van2011targeted}.
Some related papers look to combine various sources of information,
for instance from randomized control trials and observational data 
to estimate the ATE \citep{stuart2011use, hartman2015sample}.
Other papers leverage machine learning techniques, 
such as random forests, 
for estimating causal effects~\citep{alaa2017bayesian, wager2018estimation}.


Some papers investigate ATE estimation 
with confounded data by leveraging 
mediators~\citep{pearl1995causal}
and proxies~\citep{miao2018identifying}.
Others investigate
combining confounded observational data 
with \emph{experimental} 
data. 
%
%
\cite{kuroki2014measurement} identify graphical structures under which causal effect can be identified. 
\cite{miao2018identifying} 
propose to use two different types of proxies
to recover causal effects with one unobserved confounder. 
\cite{shi2018multiply} extend the work by 
\cite{miao2018identifying} to multiple confounders.
However, both methods require knowledge of proxy categories a priori 
and are not robust under misspecification of proxy categories. 
\cite{louizos2017causal} use variational autoencoders 
to recover the causal effect under the model where 
when conditioned on the unobserved confounders, 
the proxies are independent of treatment and outcome. 
%
%
\cite{pearl1995causal} introduces the front-door adjustment,
expressing the causal effect as a functional
that concerns only the (possibly confounded) treatment and outcome, 
and an (unconfounded) mediator that transmits the entire effect.


In other work,
\cite{bareinboim2013general} propose to combine observational and experimental data
under distribution shift, 
learning the treatment effect from the experimental data and transporting it to the confounded observational data 
to obtain a bias-free estimator for the causal effect. 
%
%
%
Recently, \cite{kallus2018removing} 
propose a two-step process to remove hidden confounding 
by incorporating experimental data. 
%
%
Lastly, few papers provide finite sample guarantees for causal inference.
\cite{shalit2017estimating} upper bound the estimation error 
for a family of algorithms that estimate causal effects 
under the ignorability assumption.

Unlike most prior work, we 
(i) address confounded and deconfounded (but not experimental) observational data, 
(ii) perform finite sample analysis to quantify
the relative benefit of additional confounded and deconfounded data  
towards improving our estimate of the average treatment effect, and (iii) investigate sample-efficient policies for selective deconfounding.
\section{Methods and Theory}
\label{sec:methods}

Let $T$ and $Y$ be random  variables 
denoting the 
treatment and outcome. 
We restrict these to be binary,
viewing $T$ as an indicator of whether a particular treatment has occurred
and $Y$ as an indicator of whether the outcome was \emph{successful}.
In this work, we assume the existence of a single (possible) confounder,
denoted $Z$, which can take up to $k$ categorical values (Figure \ref{fig:causal_diagram}). 
In addition, although we only include 
one unobserved confounder in our model, 
because our variables are categorical,
(as shown in Section~\ref{sec:assumptions})
this subsumes scenarios with multiple categorical confounders.
Following Pearl's nomenclature~\citep{pearl2000causality}, let 
\[
P(Y=y|\Do(T=t)) := \sum_{z\in[k]} P_{Y|T,Z}(y|t,z)P_Z(z).
\]
Our goal is to estimate the ATE, 
which can be expressed via the back-door adjustment
in terms of the joint distribution $P_{Y,T,Z}$ on $(Y,T,Z)$, 
as:
\begin{align*}
&\ATE:= P(Y=1|\Do(T=1)) - P(Y=1|\Do(T=0)), 
\\& 
= \sum_{z\in[k]} \left(P_{Y|T,Z}(1|1,z)- P_{Y|T,Z}(1|0,z)\right)P_Z(z).
\numberthis \label{equ:def}
\end{align*}
Our key contribution is to 
analyze and empirically validate
methods for estimating the ATE from 
both {\em confounded} and {\em deconfounded} observations. 
In our setup,
the \emph{confounded data} contains $n$ i.i.d. samples
from the joint distribution $P_{Y,T}$ (marginalized over the hidden confounder $Z$), 
and the \emph{deconfounded data} contains $m$ i.i.d. samples
from the full joint distribution $P_{Y,T,Z}$. 
Thus, the confounded and deconfounded data 
are $(y,t)$ and $(y,t,z)$ tuples, respectively. 
Recall that here \emph{deconfounding} means
selecting a confounded data point $(y,t)$ and 
revealing the value of its confounder $z$. 
There are two ways that we can obtain $m$ \emph{deconfounded data}, one through collecting $m$ deconfounded data 
directly without using the confounded data, and the other through revealing the value of the confounder for $m$ confounded data points.
%
%
Note that given this graph, 
we cannot exactly calculate the ATE 
unless we intervene or make further assumptions on the structure of the causal graph. 
%
Recall that such interventions or graph structures 
may not be available (e.g., in the case of genetic mutation).
Furthermore,
%
when deconfounded data is 
scarce
and confounded data is comparatively plentiful,
we 
hope to improve
our ATE estimates.
\subsection{Generalizability of Our Model}
\label{sec:assumptions}
First, we note that the use of categorical (even binary) data 
is well-established in both theory \citep{bareinboim2013general} 
and application \citep{knudson2001two,rayner2016panoply},
and not merely a simplifying proxy for continuous data.
%
Next, we show that our model subsumes scenarios 
with multiple categorical confounders.
First, absent additional distributional assumptions, 
our model captures multiple unobserved confounders 
by simple concatenation 
(since we impose no limit on the number of classes)
{\em without loss}. 
Now, one could make additional assumptions
(indeed, a {\em high}-dimensional setting might necessitate such assumptions) 
that could render alternative algorithms applicable. 
However, there exist many applications where 
(a) the confounder is of moderate dimension;
and (b) a practitioner would be dubious 
of any additional assumption \citep{bates2020causal}. 
Second, although in this scenario we implicitly assume
that the set of confounders is either 
never observed or entirely observed, 
this is also without loss so long as
the costs of revealing each confounders 
are the same (e.g. the genetic example). 
Intuitively, because we do not impose 
any independence assumption on the set of confounder,  
revealing all confounders offers maximal information 
on the joint distribution of the confounders.
We formalize this statement in Appendix~\ref{app_multiConfounder}.
%
%
While for simplicity we focus only on the setting in which
our confounder can be retroactively observed,
as we show in Appendix~\ref{app:pretreatment} our 
model can be 
applied 
straightforwardly 
to handle
a set of additional pretreatment covariates.
%
\subsection{Infinite Confounded Data}\label{sec:infinite}
In this subsection, 
we address the setting
where we have an {\em infinite} amount of confounded data ($n = \infty$), 
i.e., the marginal distribution $P_{Y,T}$ is known exactly. We leave the analysis on \emph{finite}  confounded data to Appendix~\ref{app:finite}.

\paragraph{Deconfounded Data Alone}
We begin with the baseline approach of using only the deconfounded data.
Let $p_{yt}^z = P_{Y,T,Z}(y,t,z)$, and let $\hat{p}_{yt}^z$ be 
empirical estimate of $p_{yt}^z$ from the deconfounded data 
using the Maximum Likelihood Estimator (MLE).
Let $\widehat{\ATE}$ be the estimated average treatment effect 
calculated by plugging $\hat{p}_{yt}^z$'s into Equation (\ref{equ:def}).
In the following theorem, 
we show a quantity of deconfounded samples $m$ which is sufficient 
to estimate the ATE to within a desired level of accuracy 
under the estimation process described above. 
Let $C = {12.5k^2\ln({8k}/{\delta})}{\epsilon^{-2}}$ throughout.

\begin{restatable}{theorem}{thmbase} {\red (Upper Bound)}
\label{thm_m0}
Using deconfounded data alone, 
$P\left(\left|\widehat{\ATE}-\ATE\right|\geq \epsilon\right)<\delta$ 
is satisfied if the deconfounded sample size $m$ is at least 
\begin{align*}
m_{\mathrm{base}} &:=
\max_{t,z} {C}{\left(\sum_y p_{yt}^z\right)^{-2}} 
=  \max_{t,z}
\frac{1}{P_{T,Z}(t,z)^2}C.
\end{align*}
\end{restatable}
The proof of Theorem~\ref{thm_m0} (Appendix~\ref{proof:thm_m0})
relies on
an additive decomposition of 
ATE estimation error
in terms of the estimation errors on the $p_{yt}^z$'s, 
along with concentration via Hoeffding's inequality.
We will contrast
Theorem \ref{thm_m0}
with counterpart 
methods
that use confounded data.

\paragraph{Incorporating Confounded Data} 
Estimating the ATE requires estimating the entire distribution $P_{Y,T,Z}$. 
To assess the utility of confounded data,
we decompose $P_{Y,T,Z}$ into two components:
(i) the confounded distribution $P_{Y,T}$; and
(ii) the conditional distributions
$P_{Z|Y,T}$.
Given infinite confounded data, 
the confounded distribution $P_{Y,T}$ is known exactly,
reducing the number of free parameters in $P_{Y, T, Z}$ by three. 
The deconfounded data can then be used exclusively 
to estimate the conditional distributions $P_{Z|Y,T}$.
To ease notation,
let $a_{yt} = P_{Y,T}(y,t)$, and let
$q_{yt}^z = P_{Z=z|Y,T}(y,t)$.
Moreover, let {\bf a}$:= (a_{00}, a_{01}, a_{10}, a_{11})$,  and let {\bf q} denote the vector that contains $q_{yt}^z$ for all values of $Y, T$ 
and $Z$.

\begin{table*}[ht!]
    \centering
    \begin{tabular}{@{} lrr @{}}
    \toprule      & {\bf $\bf w$} &  {\bf $\bf M$} \\
    \midrule
    {\bf NSP} 
    & $\beta^{-2}C_1\max_{t}\left(\frac{ a_{1t}(\sum_y a_{y\bar t})^2}{(\sum_y a_{yt})^2}, \frac{ a_{0t}(\sum_y a_{y\bar t})^2}{(\sum_y a_{yt})^2}\right)$
    & $\beta^{-2}C \max_{t} \frac{1}{\sum_y a_{yt}} $
    \\
    {\bf USP} & $4\beta^{-2}C_1\max_{t}\left(\frac{ a_{1t}^2(\sum_y a_{y\bar t})^2}{(\sum_y a_{yt})^2},\frac{ a_{0t}^2(\sum_y a_{y\bar t})^2}{(\sum_y a_{yt})^2}\right)$
    & $4\beta^{-2}C\max_{t} \frac{\sum_{y}a_{yt}^2}{(\sum_y a_{yt})^{2}}$ 
    \\
    {\bf OWSP}  &
    $2\beta^{-2}C_1\max_{t}\left(\frac{ a_{1t}(\sum_y a_{y\bar t})^2}{\sum_y a_{yt}}, \frac{ a_{0t}(\sum_y a_{y\bar t})^2}{\sum_y a_{yt}}\right)$
    & $ {2}{\beta^{-2}}C$
    \\\bottomrule
    \end{tabular}
    \caption{Comparison between  the instance-specific lower bound ($w$) and the worst-case upper bound ($M$).}
    \label{tab:my_label}
\end{table*}

\paragraph{Hardness of The Problem}
We first show that for particular choices of the conditional distributions, 
this estimation problem can be arbitrarily hard 
for any confounded distribution $\bf a$.
In particular, we show that for every fixed confounded distribution
encoded by $\bf a$ and for any finite amount of deconfounded data $m$,
there exist two conditional distributions encoded by ${\bf q}$'s 
such that we cannot distinguish these two distributions with high probability 
while their corresponding ATE values are constant away from each other.
Let $\ATE_{\bf a}({\bf q})$ denote the value of the ATE 
when evaluated under the distributions $\bf a$ and ${\bf q}$. 
Then, we have
\begin{restatable}{proposition}{propHardness}
\label{prop:hardness}
For every $\bf a$, there exists some $\epsilon,\delta$ 
such that for any fixed number of deconfounded samples $m$, 
we can always construct a pair of $\bf q$'s, say $\bf q_1$ and $
\bf q_2$, such that no algorithm can distinguish 
these two conditional distributions
with probability more than $1-\delta$, and
their corresponding ATE values are $\epsilon$ away:
$\left|\ATE_{\bf a}({\bf q_1}) - \ATE_{\bf a}({\bf q_2})\right|\geq\epsilon$.
\end{restatable}
Here, $\epsilon$ is a function of the confounded distribution $\bf a$,
and the values of $\bf q_1$ and $\bf q_2$ 
depend on the variable $\delta$ 
and the number of deconfounded samples, $m$. 
The proof of Proposition~\ref{prop:hardness} 
(Appendix \ref{proof:propHardness}) 
relies on constructing a pair of $\bf q_1$ and $
\bf q_2$ such that the value of 
$\left|\ATE_{\bf a}({\bf q_1}) - \ATE_{\bf a}({\bf q_2})\right|$ is constant
for all confounded distribution $\bf a$
where the entries of $\bf a$ are strictly positive.
In particular, this happens when the entries 
of the conditional distribution $\bf q$
approach to $0$.

Unless otherwise mentioned, in the rest of the paper, 
we assume that each entry of the conditional distribution, $q_{yt}^z$,
is bounded within the interval $[\beta, 1-\beta]$, for some small positive constant $\beta$.
We first provide a lower bound on the sample complexity needed 
for any algorithm and any confounded distribution:
\begin{restatable}{theorem}{thmLower}(Lower Bound)
\label{thm:general_lower}
For any estimator and sample selection policy, 
the number of deconfounded samples $m$ needed to achieve
$P\left(\left|\widehat{\ATE}-\ATE\right|\geq\epsilon\right)<\delta$ 
is at least $\Omega(\epsilon^{-2}\log(\delta^{-1}))$. 
\end{restatable}
The proof of Theorem~\ref{thm:general_lower} 
(Appendix~\ref{proof:genera_lower}) proceeds by construction.
When comparing this lower bound with the upper bounds that we will present later, 
we observe that our sample complexities are tight up to a constant.

In the rest of the section, we  first derive an upper bound 
of the sample complexity of a natural policy 
that is analogous to passive sampling (Theorem~\ref{thm_m1}).
We then derive the worst-case upper bound 
over all possible conditional distributions, $P_{Z|Y,T}$,  in Corollary \ref{cor:worstNSP}.
Next, we propose two additional sampling policies,
one of which enjoys an instance independent guarantee 
over the worst-case conditional distribution in $P_{Z|Y,T}$. 
We compare these sampling policies 
by investigating their sample complexity upper bounds (Theorem~\ref{thm_m2}), 
worst-case upper bounds (Corollary \ref{cor:worstUSPOWSP}), 
and lower bounds (Theorem \ref{thm:lower_bound}). 
Table~\ref{tab:my_label} summarizes 
our worst-case upper bound and lower bound results.
In addition, we derive a worst-case sample complexity guarantee 
of our proposed sampling policy in Theorem \ref{cor}.

Let 
$\hat{q}_{yt}^z$
be the empirical estimate of $q_{yt}^z$ from the confounded data using the MLE (where $m$ confounded data were deconfounded randomly).
Then, we will always calculate 
$\widehat{\ATE}$ 
by plugging the
$a_{yt}$'s and $\hat{q}_{yt}^z$'s into Equation (\ref{equ:def}).
%
The following theorem {\red upper} bounds 
the sample complexity 
for this
estimator (later, we refer to this sampling policy as the \emph{natural} selection policy):
%
\begin{restatable}{theorem}{thmNSP}{\red (Upper Bound)}
\label{thm_m1}
When incorporating (infinite) confounded data, 
$P(|\widehat{\ATE}-\mathrm{ATE}|\geq \epsilon)<\delta$ is
satisfied if the number of deconfounded samples $m$ is at least
\begin{align}
    m_{\mathrm{nsp}} &:= 
     \max_{t,z} \frac{C \sum_y a_{yt}}{\left(\sum_y a_{yt}q_{yt}^z\right)^{2}} 
     = \max_{t,z}   \frac{ P_T(t)}{P_{T,Z}(t,z)^2}C.
     \label{eq:nsp}
\end{align}
\end{restatable}

%
The proof of Theorem~\ref{thm_m1} is included in Appendix~\ref{proof:thm_m2}. Notably, $m_{\mathrm{nsp}}$ is less than $m_{\mathrm{base}}$ 
for any problem instance,
highlighting the value of confounded data.
%
%
%

{\red In addition, when $q_{yt}^z\in[\beta, 1-\beta]$, 
the maximum of Equation (\ref{eq:nsp}) over $\bf q$ 
is obtained at $q_{yt}^z = \beta$ for some $t,z$. 
Let $M_{\mathrm{nsp}}$ be the worst-case $m_{\mathrm{nsp}}$ 
over all possible values of $\bf q$.
Since $\min \max_{t}1/(\sum_y a_{yt})$ is achieved 
when $\sum_y a_{yt} = 1/2$, $\max_{t}1/(\sum_y a_{yt}) \geq 2$. Thus, 
\begin{restatable}{cor}{wosrtNSP} (Worst-Case Upper Bound Guarantee)
\label{cor:worstNSP}
\begin{align*}
    M_{\mathrm{nsp}}&: = \max_{\bf q}
    m_{\mathrm{nsp}}
    =  
    \max_{t} \frac{C}{\beta^{2}\sum_y a_{yt}} \geq  \frac{2C}{\beta^{2}}.
\end{align*}
\end{restatable}
}
\paragraph{Sample Selection Policies}
One important consequence of our procedure for estimating the ATE 
is that the four conditional distributions are estimated separately: 
the deconfounded data is partitioned into four groups, 
one for each $(y,t) \in \{0,1\}^2$, 
and the empirical measures $\hat{q}_{yt}^z$'s are then calculated separately. 
This means that the procedure does {\em not} 
rely on the fact that the deconfounded data 
is drawn from the exact distribution $P_{Y,T,Z}$, 
and in particular, the draws might as well have been made directly 
from the conditional distributions ${P}_{Z|Y,T}$. 
Suppose now that we can draw directly from these conditional distributions. 
This situation may arise when the confounder 
is fixed (like a genetic trait) and can be observed
retrospectively.
We now ask, given a budget for selective deconfounding samples,  
how should we allocate our samples among the four groups ($(y,t) \in \{0,1\}^2$)?

Let $\mathbf{x} = (x_{00},x_{01},x_{10},x_{11})$ denote a selection policy
with $x_{yt}$ indicating the proportion of samples allocated to group $(y,t)$, and $\sum_{yt} x_{yt}=1$. 
We consider the following three non-adaptive selection policies:
\begin{enumerate}
    \item {\bf Natural (NSP):} $x_{yt} = a_{yt} = P_{Y,T}(y,t)$---this is similar to drawing from $P_{Y,T,Z}$, {\red and is analogous to passive sampling.}
    %
    %
    \item {\bf Uniform (USP):} $x_{yt} = 1/4$. Splits samples evenly across all four conditional distributions. 
    \item {\bf Outcome-weighted (OWSP):} $x_{yt} =\frac{a_{yt}}{2\sum_y a_{yt}}$ $= P_{Y|T}(y|t)/2$.
    Splits 
    samples evenly across treatment groups ($T=0$ vs. $1$), 
    and within each treatment group,
    choosing the number of samples to be proportional to the outcome ($Y=0$ vs. $1$).
\end{enumerate}
While the particular form of OWSP appears to be the least intuitive,
{\red later we show it was in fact the unique policy that provides an instance independent guarantee when considering the worst-case $\bf q$'s.

For some fixed $\epsilon$ and $\delta$,
let $\mu_{\mathrm{nsp}}$ be the minimum number of samples needed 
to achieve $P(|\widehat\ATE - \ATE|\geq \epsilon)<\delta$ 
under the \emph{natural} selection policy over all estimators.
We similarly define $\mu_{\mathrm{usp}}$ and $\mu_{\mathrm{owsp}}$.
Then Theorem~\ref{thm_m1} provides an upper bound on $\mu_{\mathrm{nsp}}$ 
by studying the upper bound of a specific estimator. 
Before we provide an upper bound of the sample complexity 
of $\mu_{\mathrm{usp}}$ and $\mu_{\mathrm{owsp}}$, 
we first establish that $\mu_{\mathrm{nsp}}$ 
may be significantly worse than $\mu_{\mathrm{owsp}}$,
but $\mu_\mathrm{owsp}$ is never \emph{much} worse.

\begin{restatable}{theorem}{corUpper}
\label{cor}
For any fixed $\epsilon \in [0,0.5-2\beta(1-\beta)]$ and any fixed $\delta<1$, 
there exist distributions
where
$\mu_{\mathrm{owsp}}/\mu_{\mathrm{nsp}}$
is 
arbitrarily close to zero. In addition, for any estimator and every distribution, $\mu_{\mathrm{owsp}}/
\mu_{\mathrm{nsp}}\leq 2$.
\end{restatable}
The proof of Theorem~\ref{cor}
(Appendix~\ref{proof:cor}) proceeds by construction.
Note that the upper bound of $\epsilon$ in Theorem~\ref{cor} is not necessary the maximum achievable $\epsilon$. Instead it provides a range where Theorem~\ref{cor} holds.}
Next, we provide the upper bounds of 
{\red $\mu_\mathrm{usp}$ and $\mu_\mathrm{owsp}$ by analyzing} our algorithm
(analogous to Theorems \ref{thm_m0} and \ref{thm_m1}):

\begin{restatable}{theorem}{thmUSPOWSP} {\red (Upper Bound)}
\label{thm_m2}
Under the uniform selection policy,
with (infinite) confounded data incorporated, 
$P(|\widehat{\mathrm{ATE}}-\mathrm{ATE}|\geq \epsilon)<\delta$ 
is satisfied if
{\red $\mu_\mathrm{usp}$}
is at least
\begin{align*}
&
m_{\mathrm{usp}} := 
\max_{t,z} \frac{C \sum_y 4a_{yt}^2}{\left(\sum_y a_{yt}q_{yt}^z\right)^{2}}
=\max_{t,z}   \frac{4\sum_y P_{Y,T}(y,t)^2}{P_{T,Z}(t,z)^2}C.
\end{align*}\vspace{-5pt}
Similarly, 
for the outcome-weighted selection policy: \begin{align*}
m_{\mathrm{owsp}} &:= 
\max_{t,z} \frac{2C\left(\sum_y a_{yt}\right)^2}{\left(\sum_y a_{yt}q_{yt}^z\right)^{2}}
= \max_{t,z}   \frac{2}{P_{Z|T}(z|t)^2}C.  
\end{align*}
\end{restatable}
The proofs of Theorems~\ref{thm_m1} and \ref{thm_m2} (Appendix~\ref{proof:thm_m2}),
which differ from 
that
of Theorem~\ref{thm_m0}, 
require a modification to Hoeffding's inequality (Appendix, Lemma~\ref{cor4}), 
which we derive to bound the sample complexity 
of the weighted sum of two independent random variables.
Theorem \ref{thm_m2}
points to
some \emph{additional} advantages of OWSP.
First, OWSP 
has the nice property that
the sufficient number of samples, $m_{\mathrm{owsp}}$, 
does not depend on $P_{Y,T}$. 
Second,
a comparison of the quantities
$m_{\mathrm{usp}}$ and $m_{\mathrm{owsp}}$
suggests that USP is strictly dominated by OWSP, 
since $4a_{0t}^2+4a_{1t}^2-2(a_{0t}+a_{1t})^2=2(a_{0t}-a_{1t})^2\geq 0$. 
We might hope for a similar result by comparing $m_{\mathrm{owsp}}$ 
with $m_{\mathrm{nsp}}$ from comparing Theorems \ref{thm_m1} and \ref{thm_m2},
but
neither strictly dominates the other.
Instead, {\red recall that Theorem~\ref{cor}} shows that $\mu_\mathrm{nsp}$
may be significantly worse than $\mu_\mathrm{owsp}$,
but $\mu_\mathrm{owsp}$ is never \emph{much} worse.
{\red
Similar to Corollary~\ref{cor:worstNSP}, we now derive 
an equivalent corollary for Theorem~\ref{thm_m2} 
where we consider the worst case over  $\bf q$'s. 
\begin{figure*}[ht!]
    \centering
    \begin{minipage}{.695\textwidth}
    \includegraphics[width=0.48\textwidth]{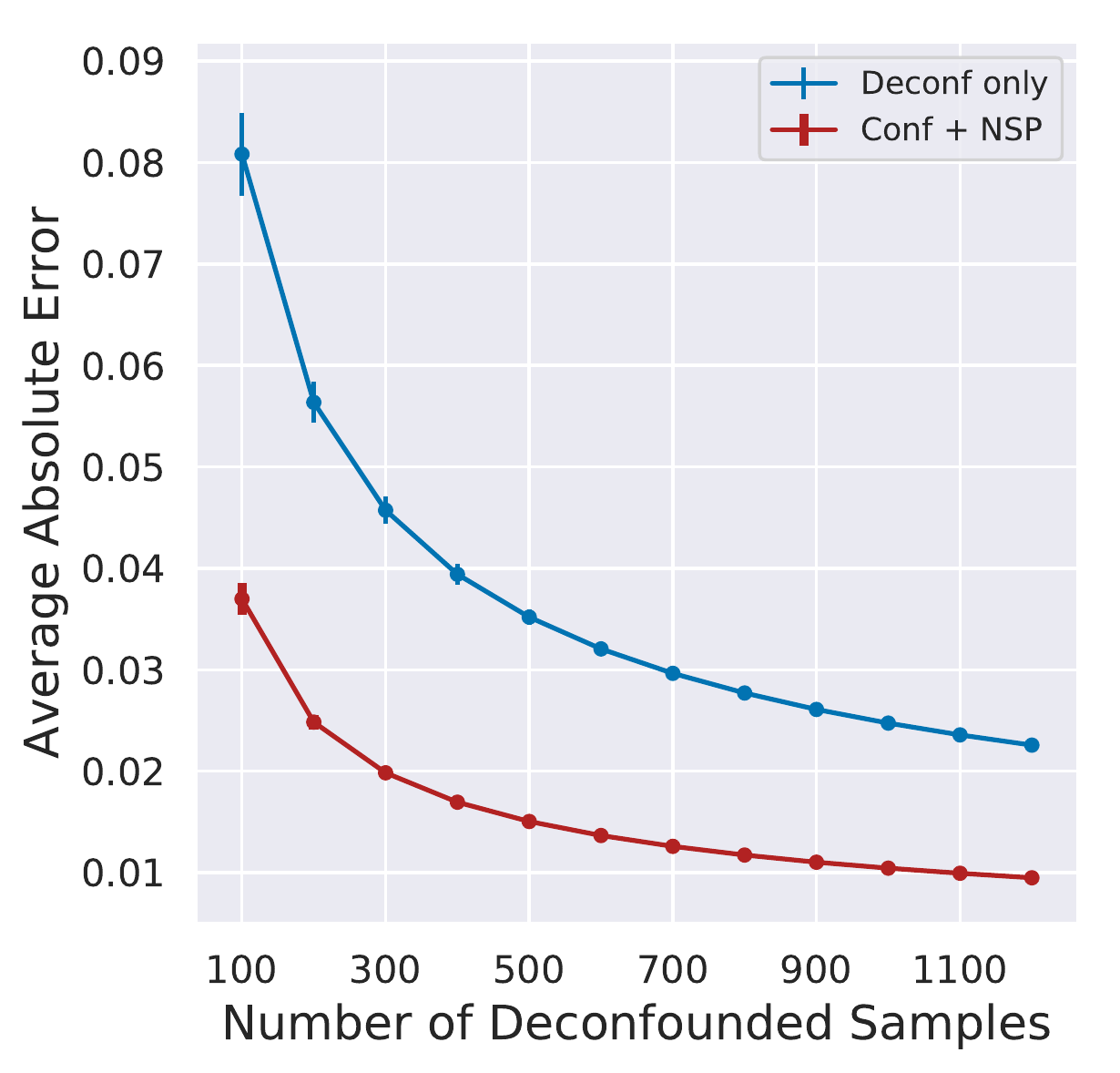}
    \includegraphics[width=0.49\textwidth]{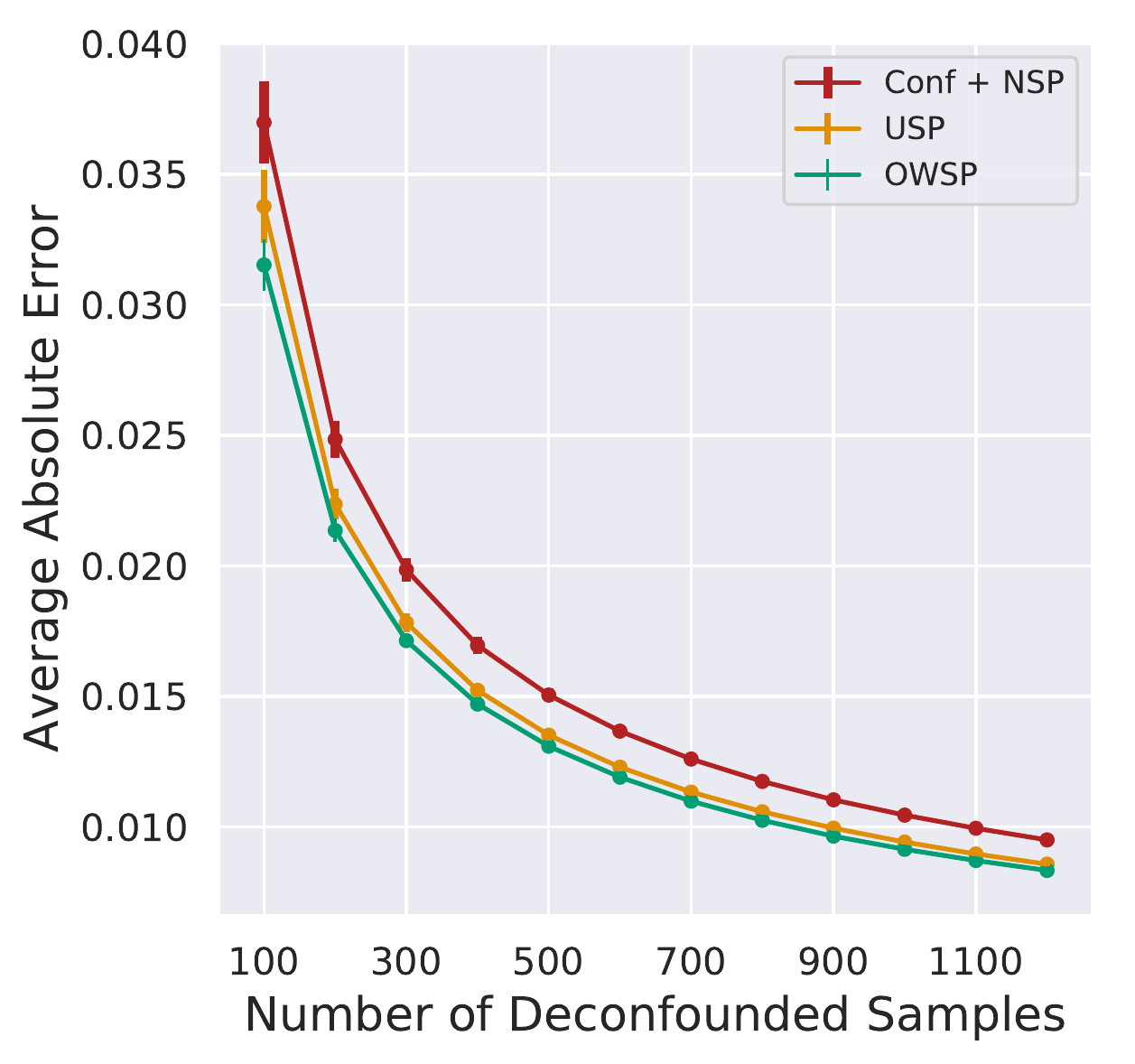}
    \end{minipage}
    \begin{minipage}{.295\textwidth}
    \centering
    \includegraphics[width=\textwidth]{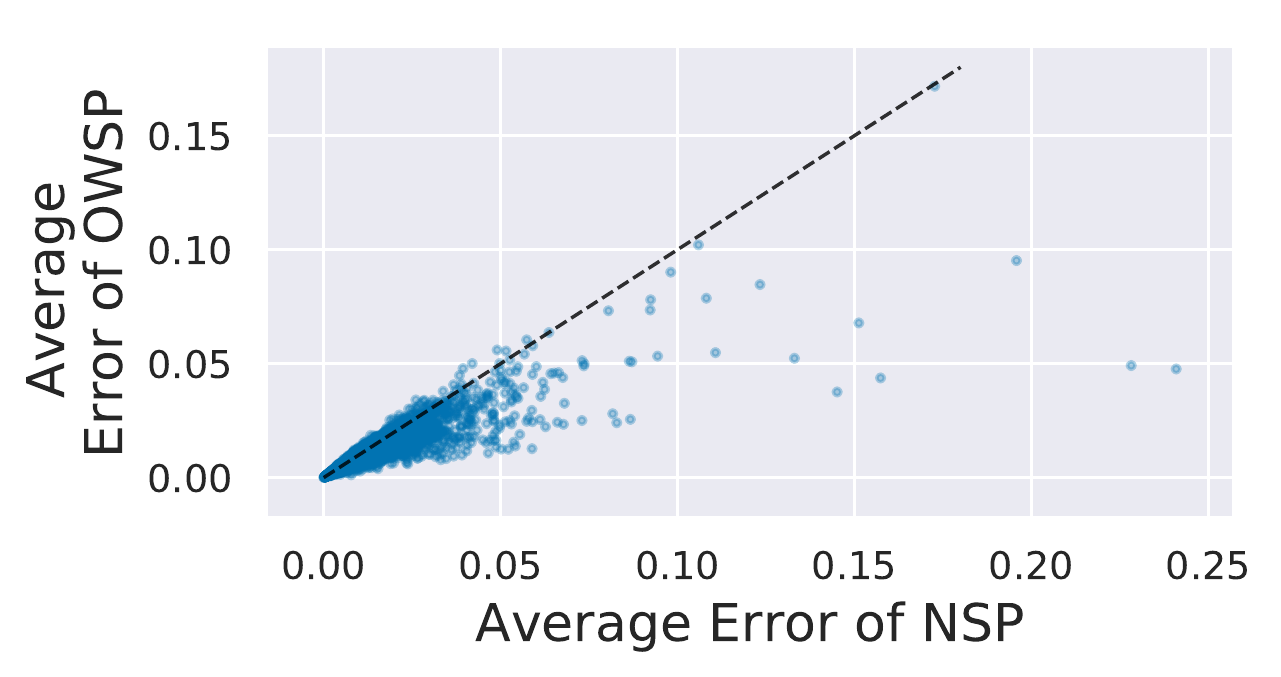}
    \includegraphics[width=\textwidth]{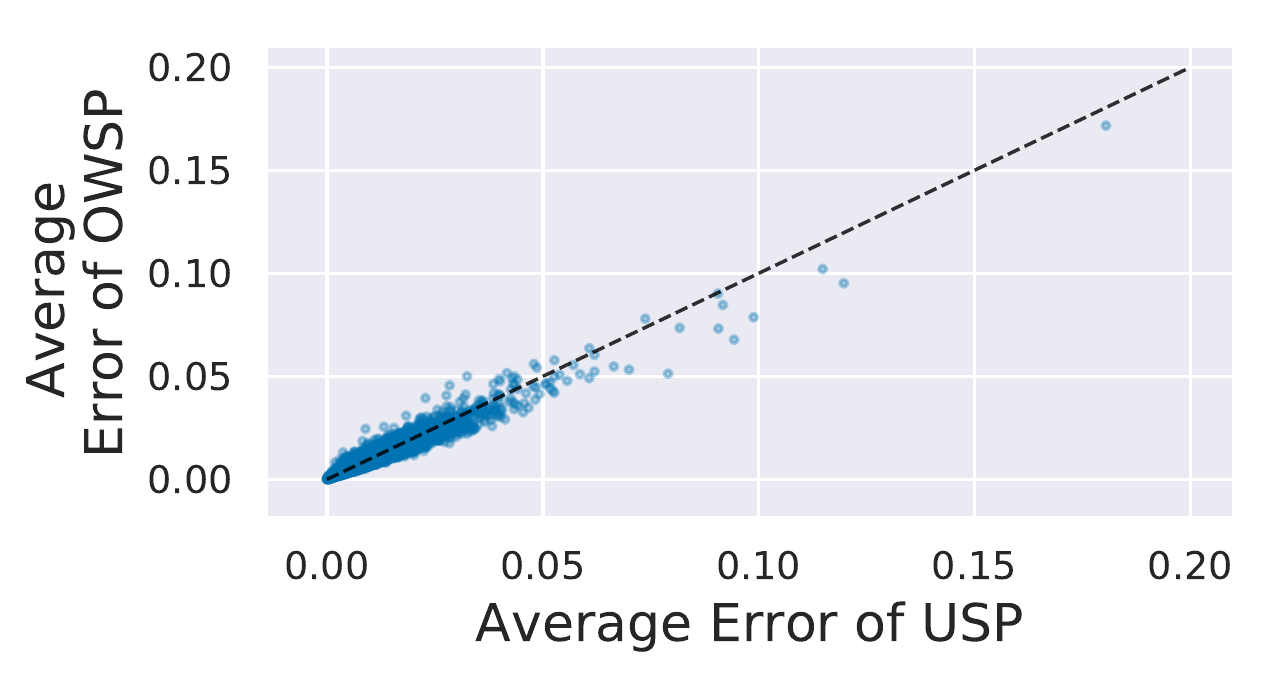}
    \end{minipage}
    \caption{Performance of the four
    policies over 13,000 distributions $P_{Y,T,Z}$, 
    given
    infinite confounded data. Left and Middle: averaged error over 
    13,000 distributions for varying numbers of deconfounded samples.
    Right: error comparison 
    (each point is a single distribution averaged over $100$ replications) 
    for 1,200 deconfounded samples.}
    \label{fig:agg} 
\end{figure*}

Let $M_{\mathrm{usp}}$ and $M_{\mathrm{owsp}}$ 
be the maximum values of $m_{\mathrm{usp}}$
and $m_{\mathrm{owsp}}$, respectively,
over all possible values of $\bf q$. 
\begin{restatable}{cor}{worstUSPOWSP} (Worst-Case Upper Bound Guarantee)
\label{cor:worstUSPOWSP}
\begin{align*}
    M_{\mathrm{usp}} &= \max_{t} \frac{4C\sum_{y}a_{yt}^2}{\beta^2(\sum_y a_{yt})^2}
    ; \;
    M_{\mathrm{owsp}}
    =\max_{t} \frac{2C}{\beta^2}\leq M_{\mathrm{nsp}}.
\end{align*}
\end{restatable}
First, note that $M_{\mathrm{owsp}}$ is independent 
of the confounded distribution $\bf a$. 
Furthermore, from the proof of Theorem~\ref{thm_m2},
we observe that OWSP is the unique policy 
that makes this upper bound independent of $\bf a$.
When comparing Corollaries~\ref{cor:worstNSP} 
and \ref{cor:worstUSPOWSP}, we observe that 
OWSP always dominates NSP when taking the worst case over $\bf q$'s.
{\red
Lastly, we provide the lower bounds of 
$\mu_\mathrm{nsp}, \mu_\mathrm{usp},$ and $\mu_\mathrm{owsp}$ 
that are analogous to Theorem~\ref{thm:general_lower}:
}

\begin{restatable}{theorem}{thmLowerPolicies}(Lower Bound)
\label{thm:lower_bound} For every $\bf a$, there exists a $\bf q$ such that 
$\mu_\mathrm{nsp}$
is at least
$$w_{\mathrm{nsp}}: = \frac{C_1}{\beta^2}\max_{t}\left(\frac{ a_{1t}(\sum_y a_{y\bar t})^2}{(\sum_y a_{yt})^2}, \frac{ a_{0t}(\sum_y a_{y\bar t})^2}{(\sum_y a_{yt})^2}\right);$$
similarly for  uniform selection policy: 
$$w_{\mathrm{usp}}:=  \frac{C_1}{\beta^2}\max_{t}\left(4\frac{ a_{1t}^2(\sum_y a_{y\bar t})^2}{(\sum_y a_{yt})^2},4\frac{ a_{0t}^2(\sum_y a_{y\bar t})^2}{(\sum_y a_{yt})^2}\right);$$
similarly for outcome-weighted sample selection policy: $$w_{\mathrm{owsp}}:= \frac{C_1}{\beta^2}\max_{t}\left(2\frac{ a_{1t}(\sum_y a_{y\bar t})^2}{\sum_y a_{yt}}, 2\frac{ a_{0t}(\sum_y a_{y\bar t})^2}{\sum_y a_{yt}}\right),$$
where $\bar t = 1-t$ and $C_1 \propto (k\beta - 1)^2\ln(\delta^{-1})\epsilon^{-2}$.
\end{restatable}
The proof  (Appendix~\ref{proof:lower_bound}) proceeds by construction. 
Table~\ref{tab:my_label} summarizes our 
worst-case upper bounds and instance-specific lower bounds. When comparing the constants $C$ and $C_1$, 
we observe that the upper bounds and lower bounds match in $k, \epsilon,$ and $\delta$, 
demonstrating the relative tightness of our analysis.

We have 
shown
the advantages of OWSP
\emph{given an infinite amount of confounded data}.
However, in practice,
the confounded data 
is
finite. In Appendix~\ref{app:finite}, 
we analyze the sample complexity upper bound 
of our algorithm under \emph{finite} confounded data. 
One new issue that arises with finite confounded data 
is that a sampling policy may not be feasible 
because there are not enough confounded samples to deconfound. 
In our experiments,
when this happens,
we approximate the target sampling policy as closely 
as is feasible (see Appendix~\ref{procedure}). 


}

\section{Experiments}
\label{sec:experiments}
Since the upper bounds 
that we derived in Section~\ref{sec:methods} are not necessarily tight, we first perform synthetic experiments to assess the tightness of our bounds. 
For the purpose of illustration, we focus on binary confounders $Z$ throughout this section, and denote $q_{yt} = P_{Z=1|Y,T}(y,t)$. We first compare the sampling policies in synthetic experiments on randomly chosen distributions $P_{Y,T,Z}$, 
measuring both the average and worst-case performance of each sampling policy.
We then measure the effect of having finite (vs. infinite) confounded data. Finally, we test the performance of OWSP on real-world data taken from a genetic database, COSMIC, that includes genetic mutations of cancer patients
\citep{tate2019cosmic, cosmic2019}.
Because this is (to our knowledge) the first paper
to investigate the problem of \emph{selective deconfounding},
the methods in described Section \ref{sec:lit} are not directly comparable to ours.
\begin{figure*}[t]
    \centering
    \begin{minipage}{0.28\textwidth}
    \includegraphics[width=\linewidth]{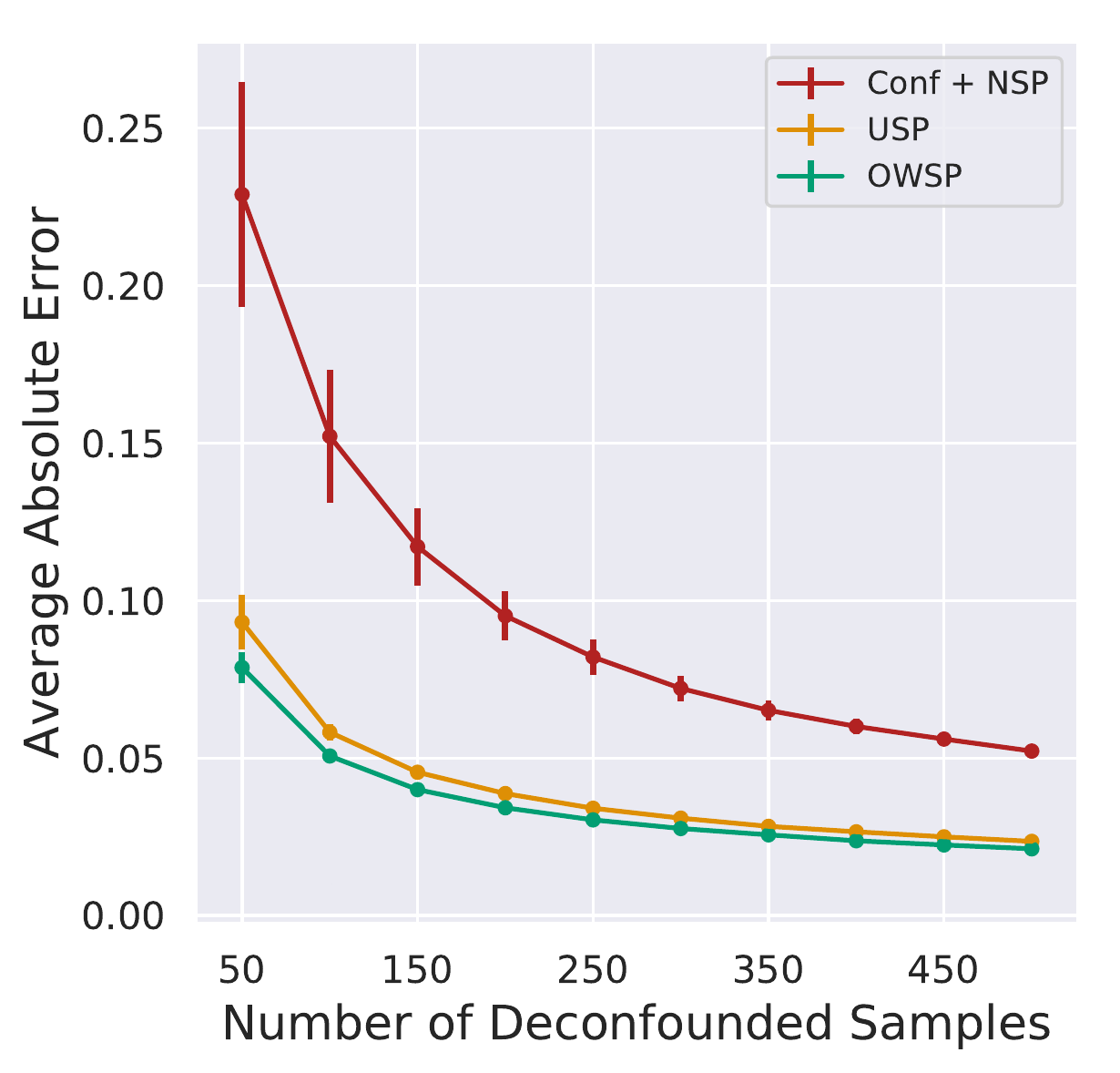}
    \includegraphics[width=\linewidth]{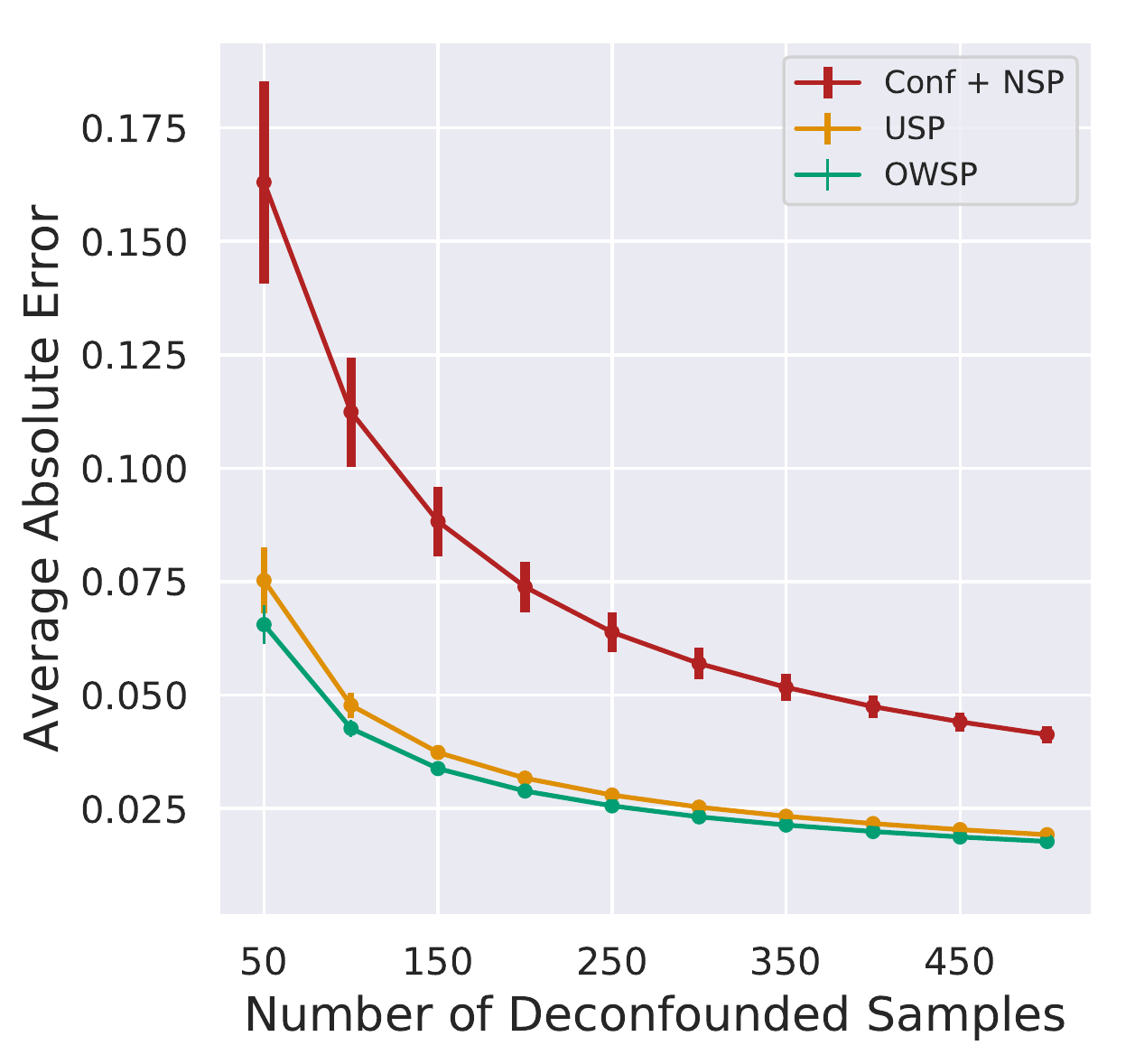}
    \end{minipage}
    \begin{minipage}{0.28\textwidth}
    \includegraphics[width=\linewidth]{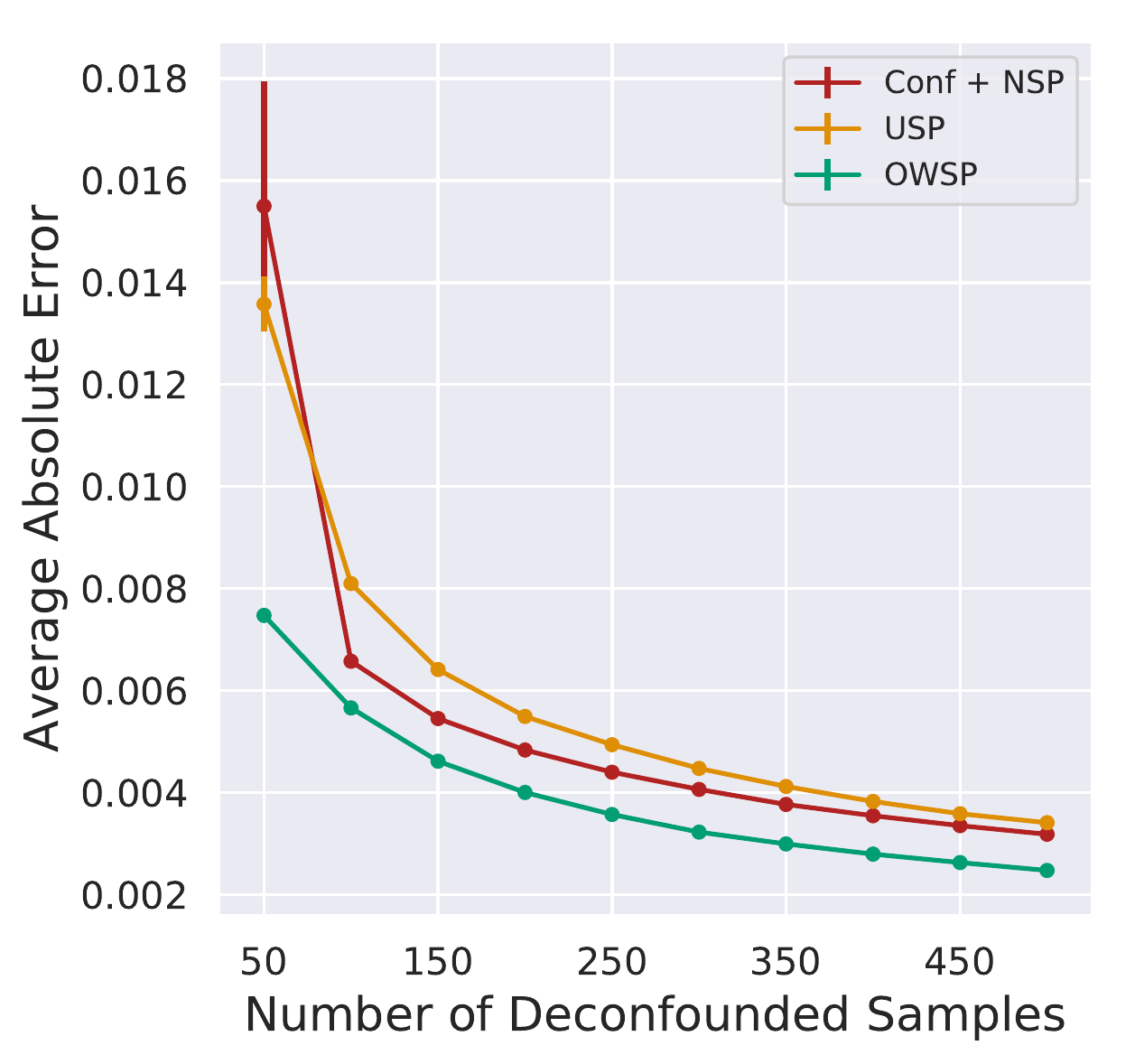}
    \includegraphics[width=\linewidth]{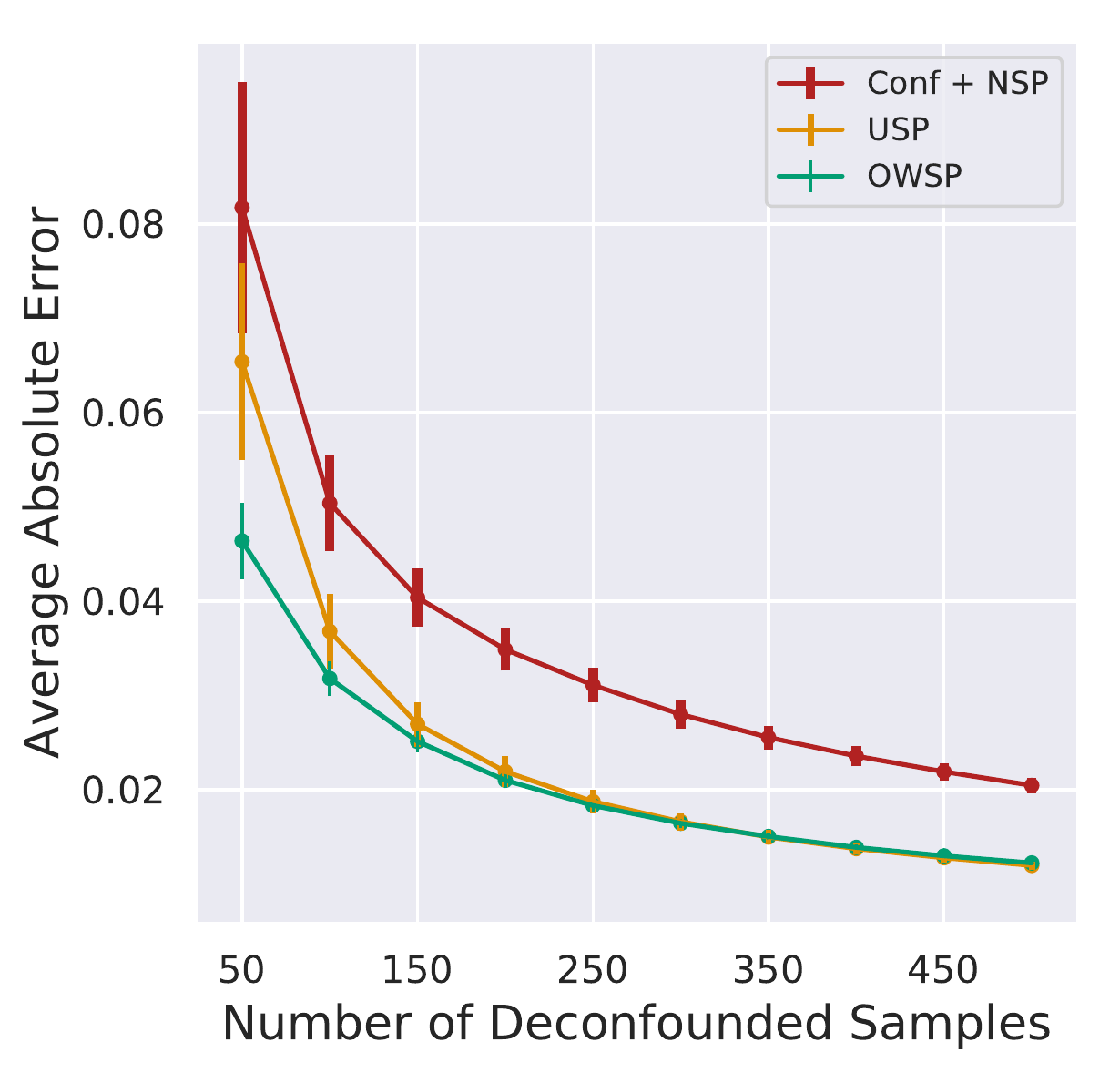}
    \end{minipage}
    \begin{minipage}{0.28\textwidth}
    \includegraphics[width=\linewidth]{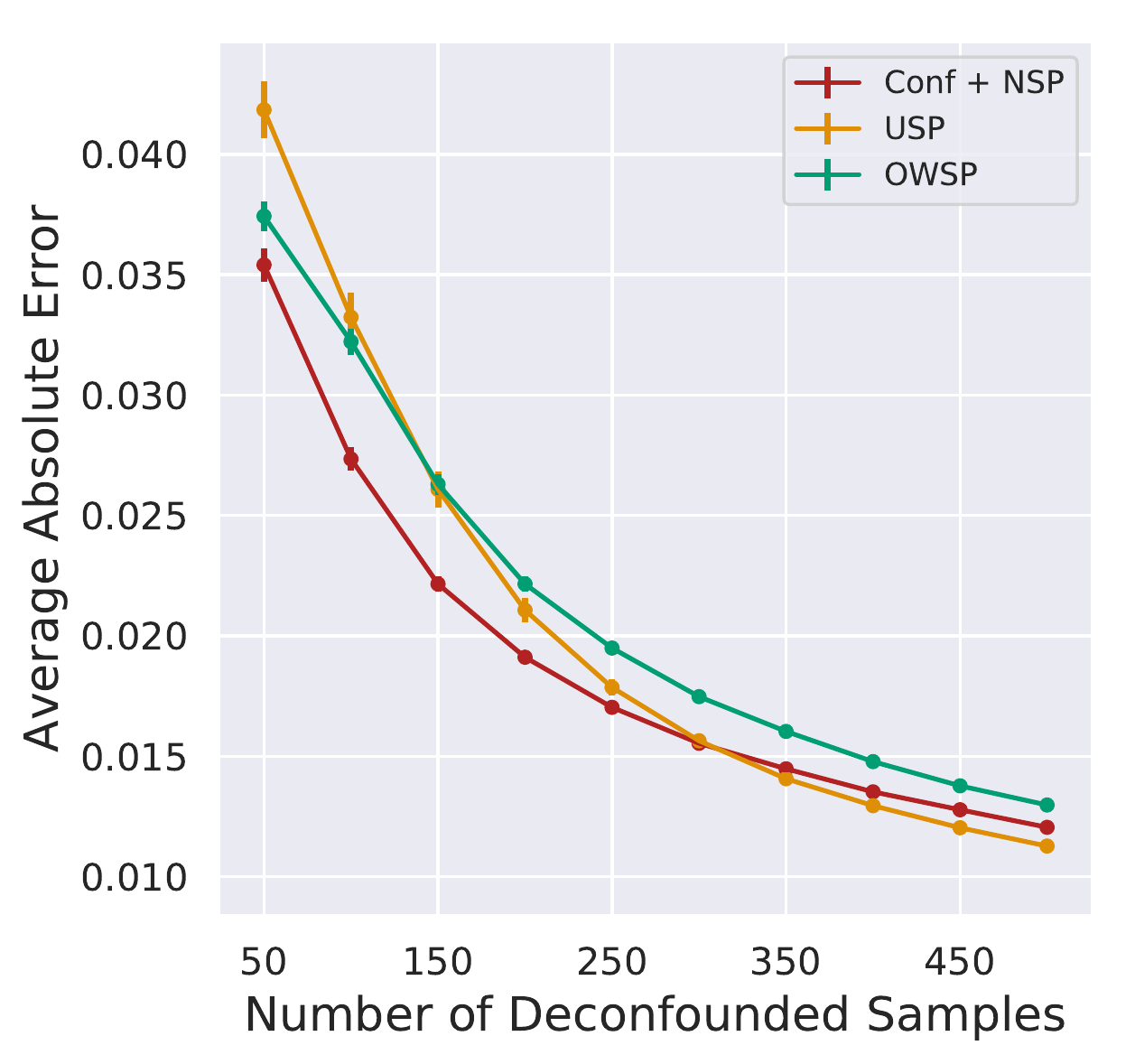}
    \includegraphics[width=\linewidth]{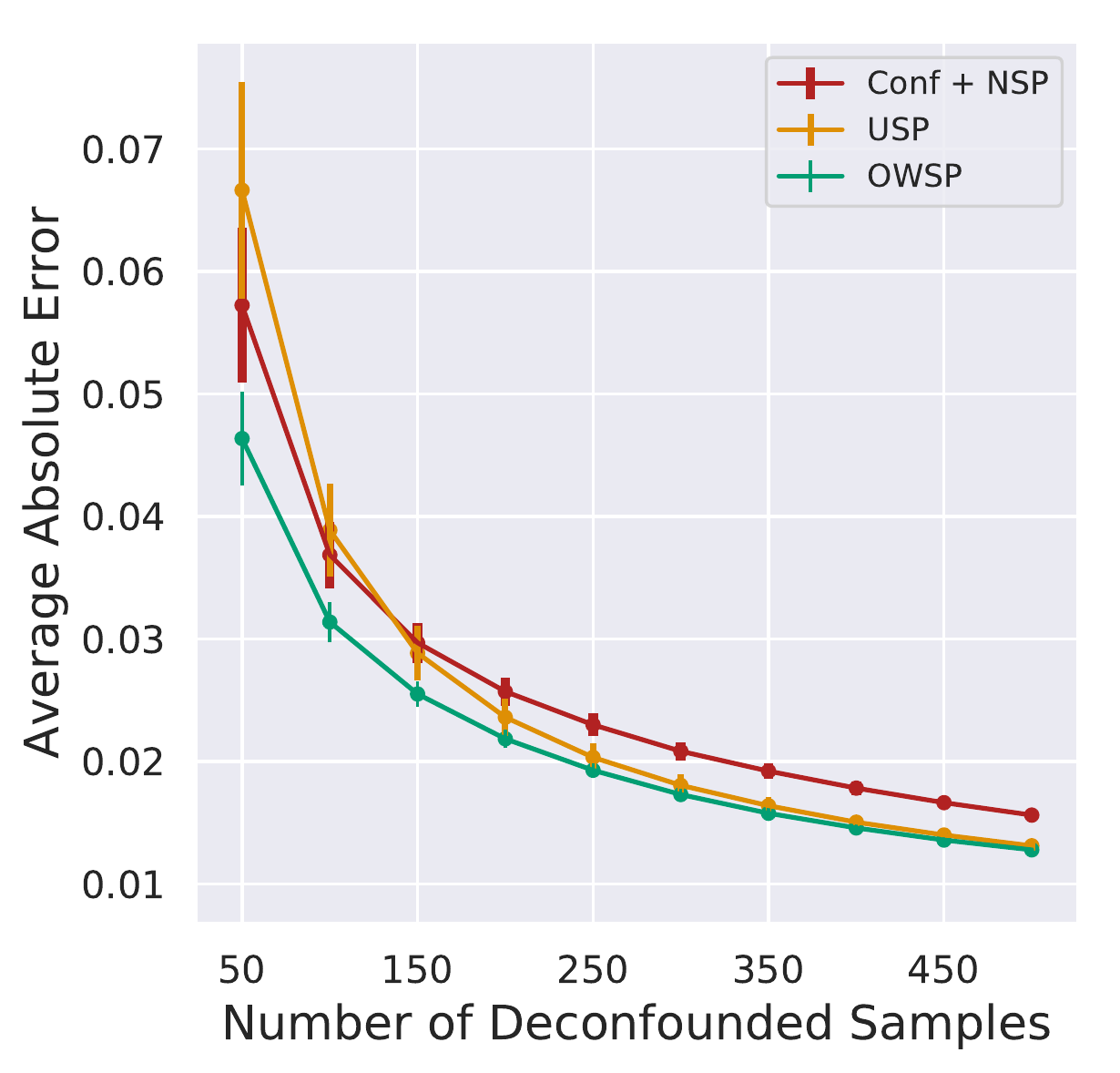}
    \end{minipage}
    \caption{Comparison of selection policies for adversarially chosen instances. Top (left) NSP performs the worst: $\mathbf{a}=(0.9,0.02,0.01,0.07)$ and $\mathbf{q}=(0.9,0.7,0.01,0.3)$. Top (middle) USP performs the worst: $\mathbf{a}=(0.79,0.01,0.02,0.18)$ and $\mathbf{q}=(0.5,0.01,0.05,0.5)$. Top (right) OWSP performs the worst: $\mathbf{a}=(0.5,0.01,0.19,0.3)$ and $\mathbf{q}=(0.05,0.5,0.055,0.4)$. Bottom: 
    the same $\mathbf{a}$'s 
    but averaged over $500$ $\mathbf{q}$'s drawn uniformly from $[0,1]^4$.}
    \label{examples}
\end{figure*}

\subsection{Infinite Confounded Data
}
Assuming access to infinite confounded data, we experimentally evaluate all four sampling methods 
for estimating the ATE: 
using deconfounded data alone, and using confounded data 
that has been selected according to NSP, USP, and OWSP. 
Let 
$\mathbf{a}:=\left(a_{00},a_{01},a_{10},a_{11}\right), \text{ and } \mathbf{q}:=\left( q_{00},q_{01},q_{10}, q_{11} \right),$ encoding the confounded and conditional distributions, respectively.
We evaluate the performance of four methods in terms of the {\em absolute error}, $|\widehat{\ATE} - \ATE|$. 
Because the variance of our estimators 
cannot be analyzed in closed form,
we report the variance of the \emph{absolute error} 
averaged over different instances in terms of the error bar in the figures.


\paragraph{
Randomly Generated Instances }
We first evaluate the four methods over a randomly selected set of distributions.
Figure~\ref{fig:agg} was generated by averaging  
over 13,000 instances, each with the distribution $P_{Y,T,Z}$ 
drawn uniformly from the
unit $7$-Simplex.
Every instance consists of $100$ replications, 
each with a random draw of 1,200 deconfounded samples. 
The absolute error is measured as a function of the number of deconfounded samples in steps of $100$ samples.
Figure \ref{fig:agg} (left) compares the use of deconfounded data along 
with the incorporation of confounded data selected naturally
(as in the comparison of Theorems \ref{thm_m0} and \ref{thm_m1}).
It shows that incorporating confounded data yields a significant improvement in estimation error. 
For example, achieving an absolute error of $0.02$ 
using deconfounded data alone requires more than 1,200 samples on average, while by incorporating confounded data, only $300$ samples are required. 
We observe that  
the variance of our estimator has decreased dramatically by incorporating infinite 
confounded 
data.
Having established the value of confounded data, 
Figure \ref{fig:agg} (middle) compares the three selection policies. 
We find that 
OWSP outperforms both NSP and USP
in terms of both the absolute error and the variance when averaged over joint distributions.
To compare the performance of our sampling policies on an instance level, we provide two scatter plots in  Figure~\ref{fig:agg} (right), each containing
the 13,000 
instances in the left figures
and averaged
over $100$ replications. 
The number of deconfounded samples is fixed at 1,200.
We observe that OWSP outperforms NSP and USP 
in the majority of instances.

\begin{figure*}
\centering
\begin{subfigure}
    \centering
    \includegraphics[width=0.26\textwidth]{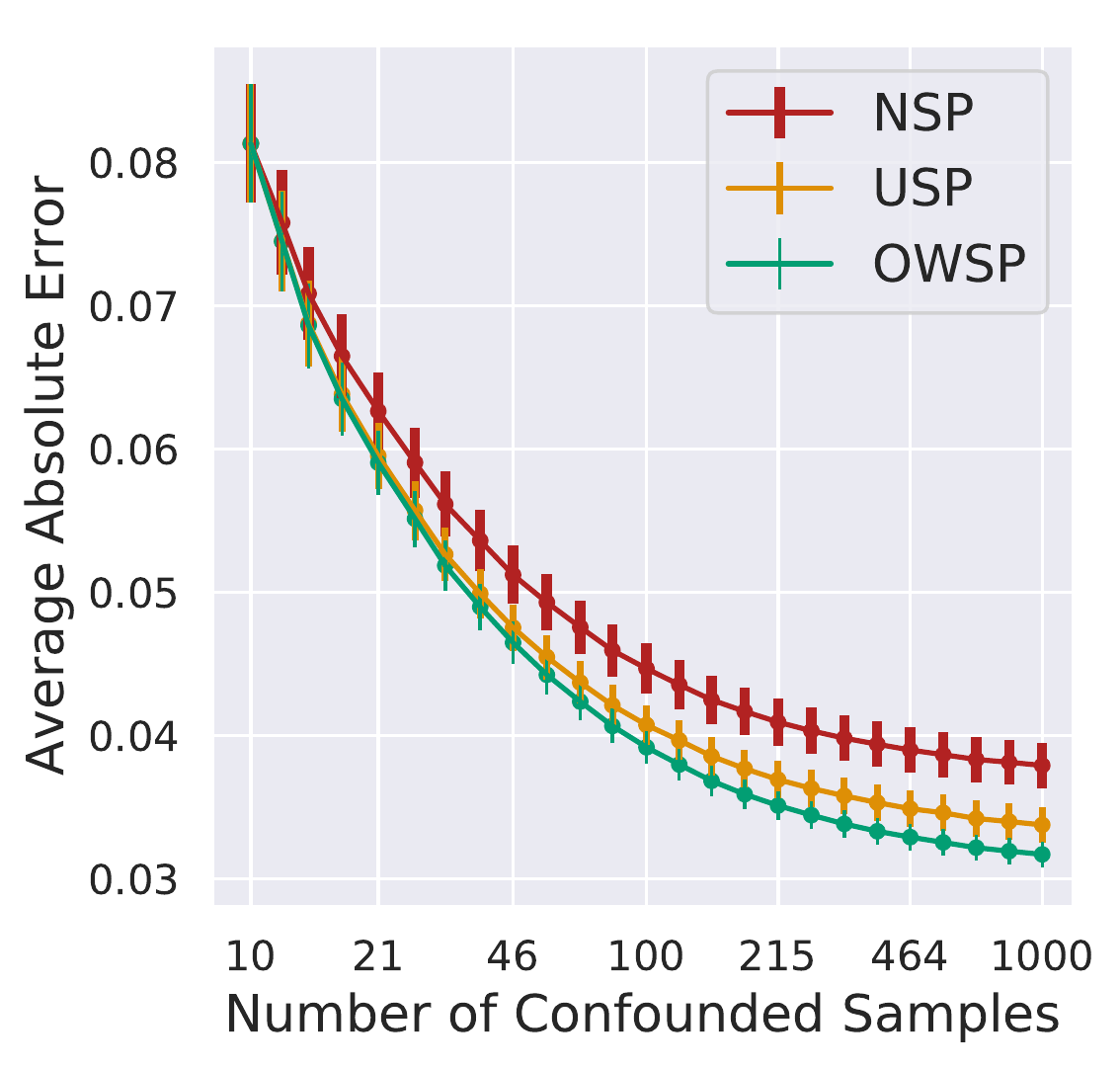}
    \end{subfigure}
\begin{subfigure}
\centering
    \includegraphics[width=0.355\textwidth]{figure/finite_scatter_NSP_OWSP681.pdf}
\end{subfigure}
\begin{subfigure}
    \centering
    \includegraphics[width=0.355\textwidth]{figure/finite_scatter_USP_OWSP681.pdf}
\end{subfigure}
    \caption{Experiment on finite confounded data over 13,000 distributions $P_{Y,T,Z}$, each averaged over $100$ replications. The number of deconfounded samples is fixed at $100$. Left: averaged over the 13,000 distributions. Middle and Right: error comparison at 681 confounded samples.}
    \label{fig:agg_finite}
\end{figure*}

\paragraph{Worst-Case Instances}
We evaluate the performance of the three selection policies 
on joint distributions chosen adversarially against each in Figure~\ref{examples}. 
%
The three sub-figures (the columns) 
correspond to instances where NSP, USP, and OWSP perform the worst, respectively, from the left to the right.
Each sub-figure 
is further subdivided: 
the top contains results for the single adversarial example
while the bottom 
is averaged over
$500$ $\mathbf{q}$'s sampled uniformly from $[0,1]^4$. 
The absolute error is averaged over 10,000 replications in the left figures and over $500$ in the right. 
In all cases, we draw $500$ deconfounded samples and measure 
the absolute error in steps of $50$ samples.
%
%
Figure~\ref{examples} (left) validates Corollary~\ref{cor}.
We observe that when the distribution of $\mathbf{a}$ is heavily skewed towards $(Y=0,T=0)$, 
OWSP and USP significantly outperform NSP. Figure~\ref{examples} (middle) shows that USP can underperform NSP, 
but when averaged over all possible values of $\mathbf{q}$, USP performs better than NSP. 
Figure~\ref{examples} (right)
shows
that OWSP can underperform NSP and USP,
but, when compared with the left and middle column, 
the performance of OWSP is close to that of NSP and USP. 
When averaged over all possible values of $\mathbf{q}$, 
OWSP outperforms both.
Moreover, OWSP's variance is the lowest across all scenarios.
Appendix~\ref{stories} provides 
examples 
in which each of these joint distributions could appear. 
\subsection{Finite Confounded Data}
Given only $n$ confounded data, we test the performance of the OWSP against NSP and USP.
In Figure~\ref{fig:agg_finite}, the absolute error is measured as a function of the number of confounded samples in step sizes that increment in the log scale from $100$ to 10,000 while fixing the number of deconfounded samples to $100$. 
Since when we only have $100$ confounded samples, the three sampling policies are identical, the error curves corresponding to NSP, USP and OWSP start at the same point on the top left corner in Figure~\ref{fig:agg_finite} (left). We observe that as the number of confounded samples increases, OWSP quickly outperforms NSP and USP on average, and the gaps between OWSP and the other two selection policies widen. 
Since we fix the number of deconfounded samples to be $100$, 1)
all three sampling policies are equivalent
when there are only $100$ confounded samples in the dataset
(i.e., we need to deconfound all $100$ confounded samples 
in all cases),
and 2) the average absolute errors of the three selection policies 
do not converge to $0$ in Figure \ref{fig:agg_finite}.
%
Figure~\ref{fig:agg_finite} (middle and right)
compare
the performance of OWSP with 
that of the NSP and USP, respectively, on an instance level.  
We observe that
OWSP dominates NSP and USP in the majority of instances
by both the absolute error and variance.
Note that if we fix the number of confounded samples
and increase the number of deconfounded samples (with $m\leq n$),
we observe that OWSP dominates USP and NSP 
when the number of deconfounded samples is small.
The gap shrinks as the number of deconfounded samples increases. When at $m=n,$ all three methods are equivalent.

\begin{figure*}
\centering
\begin{subfigure}
    \centering
    \includegraphics[width=0.303\textwidth]{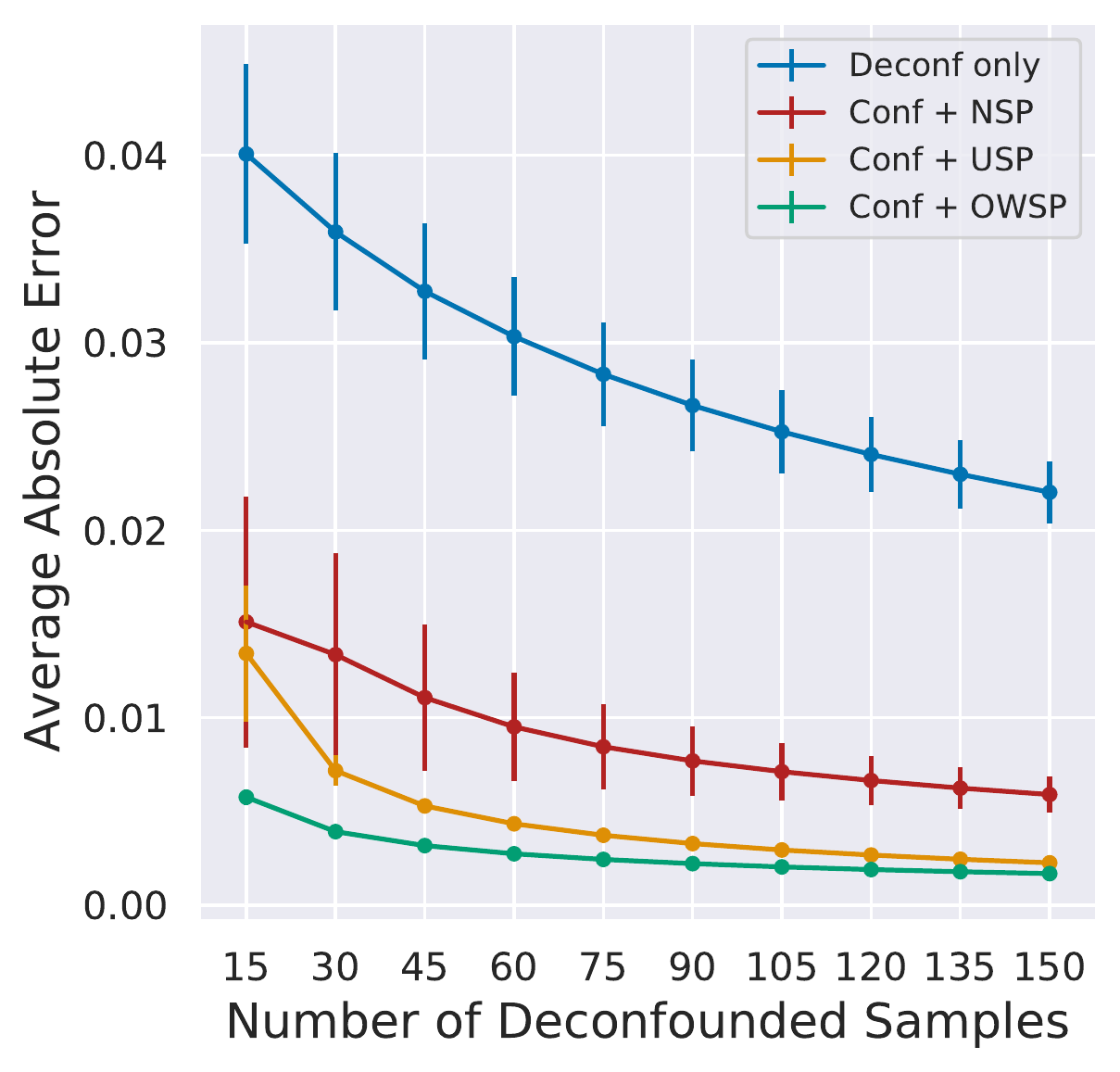}
    \end{subfigure}
\begin{subfigure}
    \centering
    \includegraphics[width=0.325\textwidth]{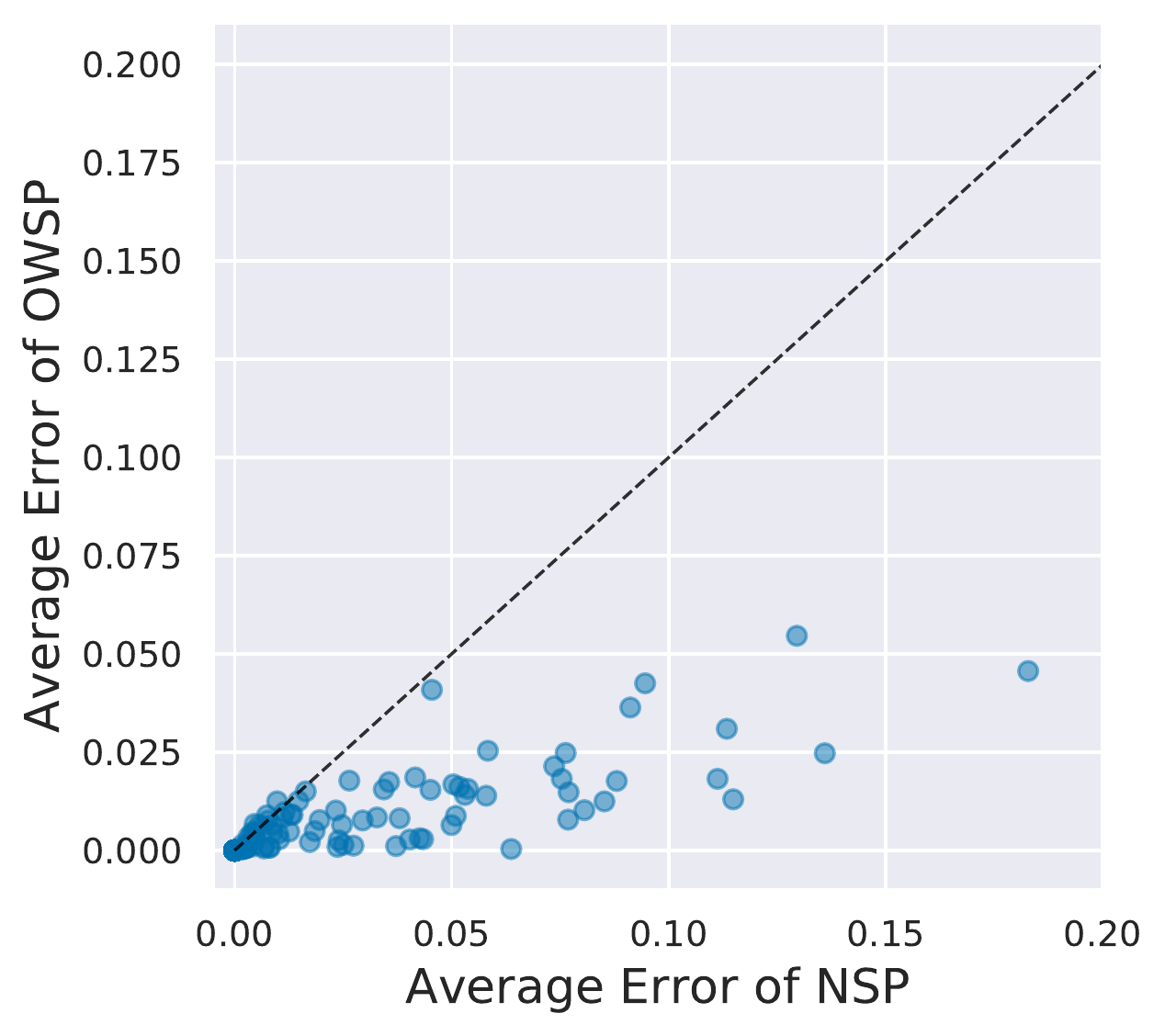}
\end{subfigure}
\begin{subfigure}
    \centering
    \includegraphics[width=0.32\textwidth]{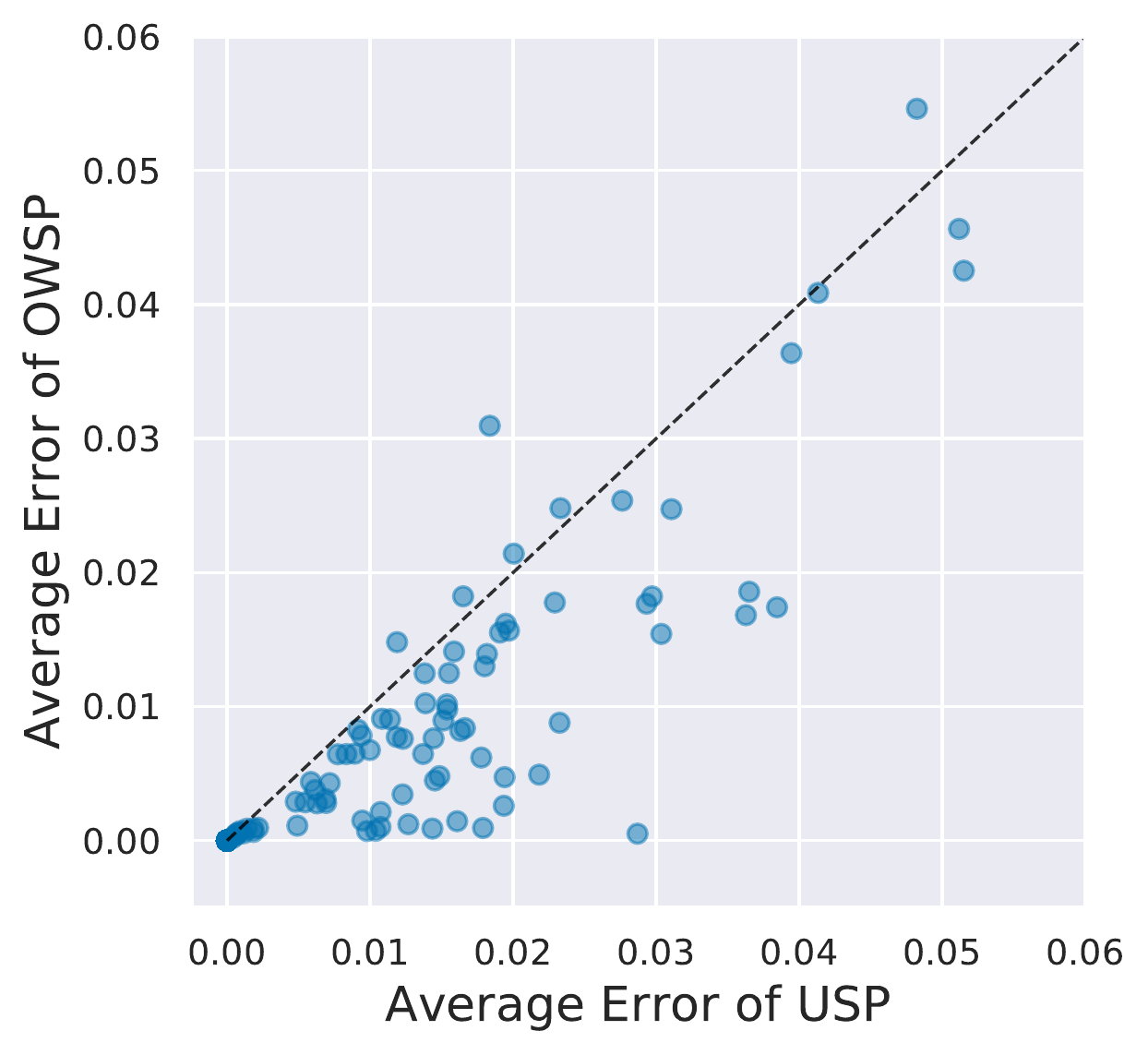}
\end{subfigure}
    \caption{Performance of the four sampling policies on the COSMIC dataset assuming infinite confounded data.  $275$ unique (cancer, mutation, mutation) combinations were extracted. Left: averaged over $275$ instances, and each averaged over 10,000 replications. Middle and Right: error comparison at $45$ deconfounded samples.}
    \label{fig:cosmic}
\end{figure*}

\subsection{Real-World Experiments: COSMIC}
\paragraph{Data  } Previously,
we chose the underlying distribution $P_{Y,T,Z}$ 
uniformly from the unit $7$-Simplex. 
However, real-world problems of interest may
not be uniformly distributed.
To illustrate the practicality of our methods,
we consider a real-world dataset,
picking three variables 
to be the outcome, treatment, and confounder, 
and artificially hiding the confounder for some samples. 
Finally, we evaluate our proposed sampling methods under 
the assumption that we have access to infinitely many confounded samples.
%
%
The Catalogue Of Somatic Mutations In Cancer (COSMIC) is a public database of DNA sequences of tumor samples.
It consists of
targeted gene-screening panels 
aggregated and manually curated over 25,000 peer reviewed papers. 
We focus on the variables: \code{primary cancer site} and \code{gene}.
Specifically, for 1,350,015 cancer patients,
we observe their 
cancer types, and for a subset of genes, whether or not a mutation was observed in each gene. 

\paragraph{Causal Models  } 
In our experiments, 
we designate cancer type as the outcome,
a particular mutation as the treatment,
and 
another mutation
as the confounder---this 
setup seems reasonable
because 
it is well known 
that multiple genetic mutations are correlated with individual cancer types \citep{knudson2001two}, 
and that mutations can cause 
both cancer itself and other mutations. 
As a concrete example, mutations in the genes 
that code RNA polymerases 
(responsible for ensuring the accuracy 
of replicating RNA)
are found to increase the likelihood 
of both other mutations 
and certain cancer types \citep{rayner2016panoply}.
The setting where the treatment mutation and cancer outcome are observed 
and the confounding mutation is unobserved is plausible because it is common that the majority of patients 
only have a subset of genes sequenced (e.g. from a commercial panel).
The top $6$ 
most commonly mutated genes were selected as treatment and confounder candidates. 
For each combination of a cancer type and two of these genes, 
we removed patients for whom these genes was not sequenced,
and kept all pairs that had at least $40$ patients 
in each of the four treatment-outcome groups (to ensure our deconfounding policies would have enough samples to deconfound).
This procedure gave us $275$ unique combinations of a cancer (outcome),  mutation (treatment), and another mutation (confounder).
%
%
Since on average, each \{cancer, mutation, mutation\} tuple contains around 125,883 patients, 
we took the estimated empirical distribution as the data-generating distribution 
and applied the ATE formula described in Section~\ref{sec:methods} to obtain the true ATE.
To model the unobserved confounder, we hid the confounding mutation parameter, only revealing it to a sampling policy when it requested a deconfounded sample.
We compared the use of deconfounded data 
along with the incorporation of confounded data under
the three sampling selection polices: NSP, USP, and OWSP. 

\paragraph{Results  } 
Figure~\ref{fig:cosmic} (left)
was generated with these $275$ instances 
each repeated for 10,000 replications. 
The absolute error is measured as a function of the number of deconfounded samples in step sizes of $15$. First, similar to Figure~\ref{fig:agg}, 
we observe that incorporating confounded data reduces 
both the absolute estimation error and the variance of the estimator
by a large margin. 
Note the improvement of OWSP over NSP is larger in this case
as compared to that seen in Figure~\ref{fig:agg}. 
Furthermore, when the number of deconfounded samples is small, OWSP outperforms USP.
Note that Figure~\ref{fig:cosmic} (left) does not start with 0 because absent any deconfounded data,
the estimated ATE is the same for all sampling policies.
%
In Figure~\ref{fig:cosmic} (middle, right), 
we fix the number of deconfounded samples to be $45$ and
compare the performance of OWSP against that of NSP and USP, respectively. 
Both figures contain the $275$ instances in the left figure,
averaged over 10,000 replications.
We observe that under this 
setup, 
OWSP dominates NSP in all instances, and 
outperforms USP in the majority of instances.


\section{Conclusion}
\label{sec:conclusions}
We propose the problem of causal inference with
\emph{selectively deconfounded} data.
{\red This problem is particularly motivated by the scenarios where interventions on treatment is not available. We theoretically analyze the upper bounds and lower bounds on}
the amount of deconfounded data required 
under each sample selection policy. 
{\red In addition, we theoretically demonstrate that the best-case gain of our proposed policy OWSP is unbounded when compared with NSP while the worst-case relative performance is bounded.}
We point to several promising directions 
for potential future research.
First, we are currently extending our analysis to the adaptive case using ideas from active learning and combinatorial optimization.
Second, we plan to extend our results to 
more general causal problems, including
linear and semi-parametric causal models.
Finally, we may extend the idea of
selective revelation of information beyond 
confounders to incorporate
mediators and proxies.


\newpage
\subsubsection*{Acknowledgments}
We wish to thank Sivaraman Balakrishnan and Uri Shalit for their valuable feedback. We would also like to thank Amazon AI, Facebook, Salesforce, the NSF, UPMC, the PwC Center, and DARPA for their support of our research.


\bibliographystyle{plainnat}
\bibliography{causal}

\newpage
\onecolumn
\appendix

\section{The Generalization of Our Models}
\subsection{Multiple Confounders}
\label{app_multiConfounder}
In this section, we show that because we do not impose any independence assumption on the set of confounder,  
revealing the values of all confounders offers maximal information 
on the joint distribution of the confounders. In particular, we will illustrate through the case where we have two binary confounders. The extension to multiple categorical confounders is straight forward.

In the case where we have two binary confounders $Z_1$ and $Z_2$, we can express the ATE as follows:
$$\ATE = \sum_{z_1, z_2} \Big(P_{Y|T,Z_1,Z_2}(1|1, z_1, z_2) - P_{Y|T,Z_1,Z_2}(1|0, z_1, z_2)\Big)P_{Z_1, Z_2}(z_1, z_2).$$
With an infinite amount of confounded data, we are provided with the joint distribution $P_{Y,T}(y,t)$. Thus, it remains to estimate the conditional distributions $P_{Z_1, Z_2|Y,T}$. In our paper, we consider only the non-adaptive policies, i.e., the number of samples to deconfound in each group $(y,t)$ is fixed a priori. In the case where the costs of revealing the values of $Z_1$ and $Z_2$ are the same and we do not have any prior knowledge on the distributions of $Z_1$ and $Z_2$, the variables $Z_1$ and $Z_2$ becomes exchangeable. In the case where the sample selection policies are completely non-adaptive (which is the case that we consider in this paper), by the symmetry of the variables $Z_1$ and $Z_2$, we have that
sampling from the joint distribution of $Z_1$ and $Z_2$ yields the maximum expected information on the value of the ATE. (Note that if the confounders take categorical values of different sizes and we allow adaptive policies, then we might be able to reduce the total cost of deconfounding to estimate the ATE to within a desired accuracy level.)

\subsection{Pretreatment Covariates}
\label{app:pretreatment}
In the case where we have known pretreatment covariates $X$, our model can be applied in estimating the individual treatment effect where we make the common ignorability assumption on the pretreatment covariates $X$ and the confounder $Z$:
given pretreatment covariates $X$ and the confounder $Z$, the values of outcome variable, $Y=0$ and $Y=1$, are independent of treatment assignment.
In this case, the distributions $P_{Y,T}(y,t)$ and $P_X(x)$ are known and the Individual Treatment Effect (ITE):
\begin{align}
    \mathrm{ITE} &= \sum_{z,x} \Big(P_{Y|T, Z, X}(1|1,z,x)-P_{Y|T, Z, X}(1|0,z,x)\Big)P_{Z,X}(z,x)\nonumber
    \\&
    =\sum_{z,x} \Big(P_{Y|T, Z, X}(1|1,z,x)-P_{Y|T, Z, X}(1|0,z,x)\Big)P_{Z|X}(z|x)P_X(x).\label{eq:covariate}
\end{align}
Note that in Equation (\ref{eq:covariate})
the only distributions we need to estimate are the conditional distributions $P_{Z|Y,T,X}$. 
The values of
$P_{Y|T,Z,X}$ and $P_{Z|X}$ can be calculated from $P_{Z|Y,T,X}$ by first conditioning the confounded distributions $P_{Y,T}$ on the values of the pretreatment covariates $X$, i.e., we first subsample all confounded (outcome, treatment) pairs for a fixed value of $X$, $X=x$, and then within
each subsample, estimate the conditional distributions $P_{Z|Y,T,X}$ by applying our methods. To obtain ITE, we weight the estimates we obtain from all subsamples by $P_X(x)$.

\section{Proofs}\label{app_proofs}
\subsection{Review of Classical Results in Concentration Inequalities}
Before embarking on our proofs, we state some classic results that we will use frequently. 
The following concentration inequalities are part of 
a family of results collectively referred to as
\emph{Hoeffding's inequality} (e.g., see \cite{vershynin2018high}).

\begin{lemma}[Hoeffding's Lemma]\label{lemma:hoeffding}
Let X be any real-valued random variable with expected value  $\mathbb{E}[X]=0$, such that $ a\leq X\leq b$ almost surely. Then, for all $\lambda \in R$,
$\mathbb{E} \left[\exp(\lambda X)\right]\leq \exp \Big(\frac {\lambda ^{2}(b-a)^{2}}{8}\Big).$

\end{lemma}

\begin{theorem}[Hoeffding's inequality for general bounded r.v.s] Let $X_1,..., X_N$ be independent random variables such that $X_i\in[m_i,M_i], \forall i$. Then, for $t>0$, we have $P\left(\left|\sum_{i=1}^N \left(X_i - \mathbb{E}[X_i]\right)\right|\geq t\right)\leq 2\exp\left(-\frac{2t^2}{\sum_{i=1}^N (M_i-m_i)^2}\right)$.
\label{hoeffding}
\end{theorem}

To begin, recall the notation introduced in Section \ref{sec:methods}: we model the binary-valued treatment, the binary-valued outcome, and the categorical confounder 
as the random variables $T \in \{0,1\}$, $Y \in \{0,1\}$, 
and $Z \in \{1,\ldots,k\}$, respectively.
The underlying joint distribution of these three random variables
is represented as $P_{Y,T,Z}(\cdot,\cdot,\cdot)$.
To save on space for terms that are used frequently,
we define the following shorthand notation:
\begin{align*}
    p_{yt}^z &= P_{Y,T,Z}(y, t, z), \\
    a_{yt} &= P_{Y,T}(y,t), \\
    q^z_{yt} &= P_{Z|Y,T}(z|y,t). 
\end{align*}
These terms appear frequently because, to estimate 
the entire joint distribution on $Y,T,Z$ (the $p_{yt}^z$'s), 
it suffices to estimate the joint distribution on $Y,T$ (the $a_{yt}$'s),
along with the conditional distribution of $Z$ on $Y,T$ 
(the $q_{yt}^z$'s): $$p_{yt}^z = a_{yt}q_{yt}^z.$$ 
Finally, let $\hat{p}_{yt}^z, \hat{a}_{yt}^z$, 
and $\hat{q}_{yt}^z$ be the empirical estimates 
of $p_{yt}^z, a_{yt}^z,$ and $q_{yt}^z$, respectively, using the MLE.

\subsection{Proof of Theorem~\ref{thm_m0}}
\label{proof:thm_m0}

\thmbase*

\begin{proof}[Proof of Theorem~\ref{thm_m0}]
This proof proceeds as follows: first, we prove a sufficient (deterministic) condition, on the errors of our estimates of $p_{yt}^z$'s, under which $|\widehat{\ATE}-\ATE|$ is small. Second, we show that the errors of our estimates of $p_{yt}^z$'s are indeed small with high probability.

\paragraph{Step 1:} First, we can write the ATE in terms of the $p_{yt}^z$'s as follows:
\begin{equation*}
\ATE = \sum \limits_z \left(P_{Y|T,Z}(1|1,z)-P_{Y|T,Z}(1|0,z)\right)P_{Z}(z)
= \sum \limits_z
\left(
\left(\frac{p_{11}^z}{\sum \limits_y p_{y1}^z} - \frac{p_{10}^z}{\sum \limits_y p_{y0}^z}\right)
\left(\sum \limits_{y,t} p_{yt}^z\right)\right).
\end{equation*}

In order for the ATE to be well-defined, we assume $\sum_y p_{yt}^z \in (0,1)$ for all $t,z$ throughout.
We can then decompose $|\widehat{\ATE}-\ATE|$:
\begin{align*}
|\widehat{\ATE}-\ATE| 
&=
\left|\sum \limits_z
\left(
\left(\frac{\hat{p}_{11}^z}{\sum \limits_y \hat{p}_{y1}^z} - \frac{\hat{p}_{10}^z}{\sum \limits_y \hat{p}_{y0}^z}\right)
\left(\sum \limits_{y,t} \hat{p}_{yt}^z\right)
-
\left(\frac{p_{11}^z}{\sum \limits_y p_{y1}^z} - \frac{p_{10}^z}{\sum \limits_y p_{y0}^z}\right)
\left(\sum \limits_{y,t} p_{yt}^z\right)\right)
\right|
\\&
\leq
\sum \limits_z 
\left|
\left(\frac{\hat{p}_{11}^z}{\sum \limits_y \hat{p}_{y1}^z} - \frac{\hat{p}_{10}^z}{\sum \limits_y \hat{p}_{y0}^z}\right)
\left(\sum \limits_{y,t} \hat{p}_{yt}^z\right)
-
\left(\frac{p_{11}^z}{\sum \limits_y p_{y1}^z} - \frac{p_{10}^z}{\sum \limits_y p_{y0}^z}\right)
\left(\sum \limits_{y,t} p_{yt}^z\right)
\right|
    .
\end{align*}
Thus, in order to upper bound $\left|\widehat{\ATE} - \ATE \right|$ by some $\epsilon$, it suffices to show that \begin{equation} \label{eqn:B11}
    \left|
\left(\frac{\hat{p}_{11}^z}{\sum \limits_y \hat{p}_{y1}^z} - \frac{\hat{p}_{10}^z}{\sum \limits_y \hat{p}_{y0}^z}\right)
\left(\sum \limits_{y,t} \hat{p}_{yt}^z\right)
-
\left(\frac{p_{11}^z}{\sum \limits_y p_{y1}^z} - \frac{p_{10}^z}{\sum \limits_y p_{y0}^z}\right)
\left(\sum \limits_{y,t} p_{yt}^z\right)
\right| \le \frac{\epsilon}{k},\;\; \forall z.
\end{equation} 
\paragraph{Step 2:}To
bound the above terms, we first 
derive Lemma~\ref{lemma1} for bounding the error of the product of two estimates in terms of their two individual errors:

\begin{lemma}\label{lemma1}
For any $ u, \hat{u}\in [-1,1]$, and $v,\hat{v}\in [0,1]$,
suppose there exists $\epsilon,\theta \in (0,1)$ such that all of the following conditions hold:
\begin{enumerate}
    \item $|u-\hat u|\leq(1-\theta)\epsilon$
    \item $|v-\hat v|\leq\theta\epsilon$
    \item $u+\epsilon\leq 1$
    \item $v+\epsilon\leq 1$
    \item $\epsilon\leq\min(u,v)$
\end{enumerate}
Then, $|uv-\hat{u}\hat{v}| \leq \epsilon$.
\end{lemma}
\begin{proof}[Proof of Lemma~\ref{lemma1}]
Since $\left|u-\hat u\right|\leq(1-\theta)\epsilon$, we have $\hat{u}\in[u-(1-\theta)\epsilon, u+(1-\theta)\epsilon]$, and similarly, from $|v-\hat v|\leq\theta\epsilon$, we have $\hat{v} \in [v-\theta\epsilon, v+\theta\epsilon]$. Thus, 
\begin{align*}
\left|uv-\hat{u}\hat{v}\right| 
&\leq \max \left(|uv-(u+(1-\theta)\epsilon)(v+\theta\epsilon)|, |uv-(u-(1-\theta)\epsilon)(v-\theta\epsilon)|\right)
\qquad\  \text{(because $v,\hat{v}\geq 0$)}
\\& 
= \max (\left|\theta u\epsilon+ (1-\theta)v\epsilon+(1-\theta)\theta\epsilon^2\right|,
\left|\theta u \epsilon+ (1-\theta)v\epsilon-(1-\theta)\theta\epsilon^2\right|)
\\& 
=\left|\theta u \epsilon+ (1-\theta)v\epsilon+(1-\theta)\theta\epsilon^2\right| \qquad\qquad\qquad\qquad\qquad\qquad\qquad
\text{ (because} \; (1-\theta)\theta\epsilon^2 > 0 \text{)}
\\&
\leq \left|\theta(u+\epsilon)\epsilon+(1-\theta)v\epsilon\right| \qquad\qquad\qquad\qquad\qquad\qquad\qquad\qquad\;\, 
\text{ (because} \; \theta\epsilon^2 > (1-\theta)\theta\epsilon^2 \text{)}
\\&
\leq \epsilon \qquad\qquad\qquad\qquad\qquad\qquad\qquad\qquad\qquad\qquad\qquad \! \text{ (because}\; u+\epsilon\in[-1,1], \;\text{and}\; v\leq 1 \text{).}
\end{align*}
\end{proof}

We can apply Lemma \ref{lemma1} directly to the terms in \eqref{eqn:B11} by setting 
\begin{align*}
    u_z &=\frac{p_{11}^z}{\sum \limits_y p_{y1}^z}- \frac{p_{10}^z}{\sum \limits_y p_{y0}^z}, \\
    \hat{u}_z &=\frac{\hat{p}_{11}^z}{\sum \limits_y \hat{p}_{y1}^z}- \frac{\hat{p}_{10}^z}{\sum \limits_y \hat{p}_{y0}^z}, \\
    v_z &= \sum \limits_{y,t} p_{yt}^z, \\
    \hat{v}_z &= \sum \limits_{y,t} \hat{p}_{yt}^z,
\end{align*}
and noting that 
$u_z, \hat u_z\in [-1,1]$, and $v_z,\hat v_z\in [0,1]$. 
Lemma~\ref{lemma1} implies that the upper bound in \eqref{eqn:B11} holds if, for some $\theta \in (0,1)$, we have
$$\left|v_z-\hat{v}_z\right|
< \frac{\theta}{k} \epsilon \;\; \text{ and } \;\; |u_z-\hat{u}_z| < \frac{1-\theta}{k}\epsilon. $$ 

While we can apply standard concentration results to the $|v_z-\hat v_z|$ terms, 
the $|u_z-\hat u_z|$ terms will need to be further decomposed:
\begin{align*}
    |u_z-\hat u_z| 
    &= \left|\frac{p_{11}^z}{\sum \limits_y p_{y1}^z}- \frac{p_{10}^z}{\sum \limits_y p_{y0}^z} -\frac{\hat{p}_{11}^z}{\sum \limits_y \hat{p}_{y1}^z} + \frac{\hat{p}_{10}^z}{\sum \limits_y \hat{p}_{y0}^z}\right| \\
    &\le \left|\frac{p_{11}^z}{\sum \limits_y p_{y1}^z} -\frac{\hat{p}_{11}^z}{\sum \limits_y \hat{p}_{y1}^z} \right| +
    \left| \frac{p_{10}^z}{\sum \limits_y p_{y0}^z} - \frac{\hat{p}_{10}^z}{\sum \limits_y \hat{p}_{y0}^z}\right|.
\end{align*}
It will suffice to show that for each $t$ and $z$,
\begin{equation}\label{eqn:B12}  
\left| \frac{p_{1t}^z}{\sum \limits_y p_{yt}^z} - \frac{\hat{p}_{1t}^z}{\sum \limits_y \hat{p}_{yt}^z}\right| < \frac{1-\theta}{2k} \epsilon. \
\end{equation}

\paragraph{Step 3:}To bound these terms,
we derive Lemma~\ref{lemma2}. 
Recall that 
$p_{1t}^z+p_{0t}^z, \hat{p}_{1t}^z + \hat{p}_{0t}^z \in (0,1)$. 
\begin{lemma}\label{lemma2}
For any $w+s, \hat{w} + \hat{s} \in (0,1)$, if $|w+s-\hat{w} -\hat{s}|\leq (w+s) \epsilon$ and $\left|w-\hat{w}\right|\leq (w+s)\epsilon$,
then $$\left|\frac{w}{w+s} - \frac{\hat{w}}{\hat{w}+\hat{s}}\right|\leq 2\epsilon.$$
\end{lemma}
\begin{proof}[Proof of Lemma~\ref{lemma2}]
First, since $\left|w+s-\hat{w} -\hat{s}\right|\leq (w+s) \epsilon$, we have that $$\left|\frac{w+s}{\hat{w} +\hat{s}}-1\right|\leq \frac{w+s}{\hat{w} +\hat{s}} \epsilon,$$ or equivalently, 
$$ 1-\frac{w+s}{\hat{w} +\hat{s}}\epsilon\leq \frac{w+s}{\hat{w} +\hat{s}}\leq 1+\frac{w+s}{\hat{w} +\hat{s}}\epsilon.$$
We can apply this inequality and rearrange terms as follows to conclude the proof:
\begin{align*}
\left|\frac{w}{w+s} - \frac{\hat{w}}{\hat{w} + \hat{s}}\right| & 
=\left|\frac{1}{w+s}\right|\left|w -\hat{w} \frac{w+s}{\hat{w} + \hat{s}}\right|
\\&
\leq \left|\frac{1}{w+s}\right|
\max\left(
\left|w -\hat{w}\left(1-\frac{w+s}{\hat{w} + \hat{s}}\epsilon\right)\right|
, 
\left|w -\hat{w}\left(1+\frac{w+s}{\hat{w} + \hat{s}}\epsilon\right)\right|\right)
\\&
=\left|\frac{1}{w+s}\right|
\max \left(
\left|w -\hat{w}+\frac{w+s}{\hat{w} + \hat{s}}\hat{w}\epsilon\right|,
\left|w -\hat{w} - \frac{w + s}{\hat{w} + \hat{s}}\hat{w}\epsilon\right|
\right)
\\&
=\max \left(
\left|\frac{w -\hat{w}}{w+s}+\frac{\hat{w}}{\hat{w} + \hat{s}}\epsilon\right|,
\left|\frac{w -\hat{w}}{w+s} - \frac{\hat{w}}{\hat{w} + \hat{s}}\epsilon\right|
\right)
\\&
\leq 
\left|\frac{w -\hat{w}}{w+s}\right|
+\left|\frac{\hat{w}}{\hat{w} + \hat{s}}\right|\epsilon \\&\leq\left|\frac{w+s }{w+s}\right|\epsilon
+\left|\frac{\hat{w}}{\hat{w} + \hat{s}}\right|\epsilon 
\\&
\leq 2\epsilon. 
\end{align*}
The second to last inequality follows from the assumption that $|w-\hat{w}|\leq (w+s)\epsilon$.
\end{proof}
Lemma~\ref{lemma2} implies that \eqref{eqn:B12} is satisfied if
\[
\left|p_{1t}^z - \hat{p}_{1t}^z \right|<\frac{(\sum_y p_{yt}^z)(1-\theta)}{4k}\epsilon \;\; \text{ and } \;\; 
\left|p_{1t}^z + p_{0t}^z  - \hat{p}_{1t}^z -  \hat{p}_{0t}^z \right|<\frac{(\sum_y p_{yt}^z)(1-\theta)}{4k}\epsilon.\]

\paragraph{Step 4:}We've shown above that 
$|\widehat{\ATE}-\ATE|\leq \epsilon$ is satisfied when 
$$\left|v_z - \hat{v}_z\right|<\frac{\theta}{k}\epsilon, \quad \left|p_{1t}^z - \hat{p}_{1t}^z \right|<\frac{(\sum_y p_{yt}^z)(1-\theta)}{4k}\epsilon,$$ 
and 
$$\left|p_{1t}^z + p_{0t}^z  - \hat{p}_{1t}^z -  \hat{p}_{0t}^z \right|<\frac{(\sum_y p_{yt}^z)(1-\theta)}{4k}\epsilon, \forall t,z.$$
Note that if $\forall t, \left|p_{1t}^z + p_{0t}^z  - \hat{p}_{1t}^z -  \hat{p}_{0t}^z \right| = \left| \sum_{y} p_{yt}^z - \sum_{y} \hat{p} _{yt}^z \right|<\frac{(\sum_y p_{yt}^z)(1-\theta)}{4k}\epsilon$ 
then
$$\left|v_z-\hat{v}_z\right| = \left|\sum_{y,t} p_{yt}^z - \sum_{y,t} \hat{p} _{yt}^z\right| \leq \sum_t \left|\sum_{y} p_{yt}^z - \sum_{y} \hat{p} _{yt}^z\right|
<\frac{(\sum_{y,t} p_{yt}^z)(1-\theta)}{4k}\epsilon \leq \frac{(1-\theta)}{4k}\epsilon.$$ 
Thus, to 
remove the first constraint $\left|v_z-\hat{v}_z\right|<\frac{\theta}{k}\epsilon$,
we set $$\frac{\theta}{k}\epsilon = \frac{(1-\theta)}{4k}\epsilon,$$ and obtain $\theta = \frac{1}{5}$.

\paragraph{Step 5:}To summarize so far, Lemmas \ref{lemma1} and \ref{lemma2} allow us to upper bound the error of our estimated $\ATE$ in terms of upper bounds on the error of our estimates of its constituent terms:
$$
P\left(|\widehat{\ATE}-\ATE|<\epsilon \right)
\geq 
P\left(\bigcap_{t,z} \left\{\left|
p_{1t}^z - \hat{p}_{1t}^z
\right| < \frac{\sum_y p_{yt}^z}{5k}\epsilon \right\} \bigcap_{t,z}
\left\{\left|p_{1t}^z + p_{0t}^z  - \hat{p}_{1t}^z -  \hat{p}_{0t}^z \right|<\frac{\sum_y p_{yt}^z}{5k}\epsilon
\right\}\right),
$$
or equivalently,
$$
P\left(|\widehat{\ATE}-\ATE|\geq \epsilon \right)
\leq 
P\left(\bigcup_{t,z} \left\{\left|
p_{1t}^z - \hat{p}_{1t}^z
\right| \geq \frac{\sum_y p_{yt}^z}{5k}\epsilon \right\} \bigcup_{t,z}
\left\{\left|p_{1t}^z + p_{0t}^z  - \hat{p}_{1t}^z -  \hat{p}_{0t}^z \right| \geq \frac{\sum_y p_{yt}^z}{5k}\epsilon
\right\}\right).
$$
Applying a union bound, we have
\begin{equation} \label{eqnProof1a}
P\left(|\widehat{\ATE}-\ATE|\geq \epsilon \right)
\leq 
\sum _{t,z}
P\left(\left|
p_{1t}^z - \hat{p}_{1t}^z
\right| \geq \frac{\sum_y p_{yt}^z}{5k}\epsilon \right)
+ P\left(\left|p_{1t}^z + p_{0t}^z  - \hat{p}_{1t}^z -  \hat{p}_{0t}^z \right| \geq \frac{\sum_y p_{yt}^z}{5k}\epsilon
\right).
\end{equation}
\paragraph{Step 6:}Finally, we can apply Hoeffding's inequality (Theorem~\ref{hoeffding}) to obtain the upper bound for the inequality above.
Let $X_{yt}^z$ be the random variable that maps the event $(Y=y, T=t, Z=z)\mapsto \{0,1\}$. Then, $X_{yt}^z$ is a Bernoulli random variable with parameter $p_{yt}^z$. Let $m$ denote the total number of deconfounded samples that we have. Since $\hat{p}_{yt}$ is estimated through the MLE, we have $\hat{p}_{yt}^z = \frac{\sum_{i=1}^m X_{yt}^z}{m}$. Applying Theorem~\ref{hoeffding}, we obtain:
\begin{equation} \label{eqnProof1b}
P\left(\left|\frac{\sum_{i=1}^m X_{yt}^z}{m} - p_{yt}^z\right|\geq \frac{\sum_y p_{yt}^z}{5k}\epsilon\right) \leq 2\exp\left(-2m\frac{\left(\sum_y p_{yt}^z\right)^2\epsilon^2}{25k^2}\right),\;\;\text{and}\end{equation}
\begin{equation} \label{eqnProof1c}
P\left(\left|\frac{\sum_{i=1}^m X_{1t}^z + X_{0t}^z}{m}  - p_{1t}^z - p_{0t}^z\right|\geq \frac{\sum_y p_{yt}^z}{5k}\epsilon\right) \leq 2\exp\left(-2m\frac{\left(\sum_y p_{yt}^z\right)^2\epsilon^2}{25k^2}\right).
\end{equation}
Combining \eqref{eqnProof1a}, \eqref{eqnProof1b}, and \eqref{eqnProof1c}, we have
\begin{align*}
P\left(|\widehat{\ATE}-\ATE|\geq \epsilon \right)
&\leq 
\sum _{t,z}
P\left(\left|
p_{1t}^z - \hat{p}_{1t}^z
\right| \geq \frac{\sum_y p_{yt}^z}{5k}\epsilon \right)
+ P\left(\left|p_{1t}^z + p_{0t}^z  - \hat{p}_{1t}^z -  \hat{p}_{0t}^z \right| \geq \frac{\sum_y p_{yt}^z}{5k}\epsilon\right)
\\&
\leq 4k\max_{t,z}\left(2\exp\left(-2m\frac{\left(\sum_y p_{yt}^z\right)^2\epsilon^2}{25k^2}\right)\right)
\\&
 = 8k\max_{t,z}\exp\left(-2m\frac{\left(\sum_y p_{yt}^z\right)^2\epsilon^2}{25k^2}\right) \\& \leq \delta,
\end{align*}
where the second line follows from the fact that, since $t$ is binary, there are $4k$ terms in total.
Solving the above equation, we conclude that $P(|\widehat{\ATE}-\ATE|\geq \epsilon)<\delta$ 
is satisfied when the sample size $m$ is at least 
$$m\geq  \frac{12.5k^2\ln(\frac{8k}{\delta})}{\epsilon^2} \max_{t,z} \frac{1}{\left(\sum_y p_{yt}^z\right)^2}.
%
$$
\end{proof}

\subsection{Proof of Proposition~\ref{prop:hardness}}\label{proof:propHardness}

\propHardness*

\begin{proof}[Proof of Proposition~\ref{prop:hardness}]

It suffices to show for the case where confounder takes binary value. The extension to categorical confounder is straightforward as illustrated in the proof of Theorem~\ref{thm:lower_bound} in Appendix~\ref{proof:lower_bound}. Let $q_{yt} = P(Z=1|Y=y,T=t)$.
To show that Proposition~\ref{prop:hardness}  is true, it is sufficient to show that there exist a positive constant $c$ (that depends on $\bf a$) such that for all fixed $\bf a$, there exists a pair of $\bf q$ and ${\bf q}'$ such that $\|\ATE_{\bf a}({\bf q}) - \ATE_{\bf a}({\bf q}')\|>c$, with $\bf q$ and ${\bf q}'$ close in distribution. We proceed by construction.
For fixed $\bf a$, consider the following $\bf q$ pairs: ${\bf q} = (q_{00}, 0, q_{10}, \gamma)$ and ${\bf q}' = (q_{00}, \gamma, q_{10}, 0)$. Then, we have
\begin{align*}
&\ATE_{\bf a}({\bf q}) = 
(a_{00}q_{00} + a_{10}q_{10} + a_{11}\gamma) 
+ \frac{a_{11}(1-\gamma)}{a_{11}(1-\gamma) + a_{01}}(1- a_{00}q_{00} - a_{10}q_{10} 
- a_{11}\gamma) -
\\&\frac{a_{10}q_{10}}{a_{10}q_{10}+a_{00}q_{00}}(a_{00}q_{00} + a_{10}q_{10} + a_{11}\gamma) - \frac{a_{10}(1-q_{10})}{a_{10}(1-q_{10})+a_{00}(1-q_{00})}(1- a_{00}q_{00} 
- a_{10}q_{10} - a_{11}\gamma),
\end{align*}
and similarly, we have
\begin{align*}
&\ATE_{\bf a}({\bf q}') =
\frac{a_{11}}{a_{11} + a_{01}(1-\gamma)}(1- a_{00}q_{00} - a_{01}\gamma- a_{10}q_{10} 
) -
\frac{a_{10}q_{10}}{a_{10}q_{10}+a_{00}q_{00}}(a_{00}q_{00} 
\\&
+ a_{01}\gamma + a_{10}q_{10} 
) - 
\frac{a_{10}(1-q_{10})}{a_{10}(1-q_{10})+a_{00}(1-q_{00})}
(1- a_{00}q_{00} - a_{01}\gamma
- a_{10}q_{10} ).
\end{align*}
In particular, 
\begin{equation}\label{eq:hardness}
\lim_{\gamma\rightarrow 0}\ATE_{\bf a}({\bf q}) - \ATE_{\bf a}({\bf q}') = a_{00}q_{00} + a_{10}q_{10}\leq a_{00}+a_{10},
\end{equation}
where we can choose $q_{00}$ and $q_{10}$ to be $1$.

On the other hand, we can show that the number of samples needed to distinguish $\bf q$ from ${\bf q}'$ is at least $\Omega(1/\gamma)$: since ${\bf q}$ and ${\bf q}'$ are the same in two of the entries and symmetric on the rest two, to distinguish  ${\bf q}$ and ${\bf q}'$ is to distinguish a Bernoulli random variable with parameter $0$ (denoting this variable $B_0$) from a Bernoulli random variable with parameter $\gamma$ (denoting this random variable $B_{\gamma}$). Let $f$ be any estimator of the Bernoulli random variable, and $x_i, ..., x_m$ be the sequence of $m$ observations. Then we have $|\mathbb{E}_{X \sim B_0^m}[f] - \mathbb{E}_{X \sim B_{\gamma}^m}[f]|\leq \|B_0^m - B_{\gamma}^m\|_1\leq \sqrt{2(\ln 2) \mathrm{KL}(B_0^m\|B_{\gamma}^m)}\leq 2 \sqrt{(\ln 2)\gamma m}$, where the last inequality is because when given $m$ samples,  $\mathrm{KL}(B_0^m\|B_{\gamma}^m)\leq (2\gamma\ln 2 + (1-2\gamma)\ln\frac{1-2\gamma}{1-\gamma})m \leq 2\gamma m.$ On the other hand, any hypothesis test that takes n samples and distinguishes between $H_0: X_1, ... , X_n \sim P_0$ and
$H_1 : X_1, . . . , X_n \sim P_1$ has probability of error lower bounded by $\max(P_0(1), P1(0)) \geq
\frac{1}{4}e^{-n\mathrm{KL}(P_0\|P_1)}$, where $P_0(1)$ indicates the probability that we identify class $H_0$ while the true class is $H_1$. Since $P_0(1) + P_1(0) \leq \delta$, by contradiction, we can show that $m\sim\Omega(\ln(\delta^{-1})\gamma^{-1})$.

Note that this lower bound on m can be arbitrarily large by choosing $\gamma$ to be sufficiently small.
However their ATE values stay constant away as observed in Equation (\ref{eq:hardness}). Thus, for every fixed confounded distribution encoded by $\bf a$ and fixed number of deconfounded samples $m$, we can always construct a pair of conditional distributions encoded by $\bf q$ and ${\bf q}'$ such that their corresponding ATEs are constant away while the probability that we correctly identify the true conditional distribution from $\bf q$ and ${\bf q}'$ is less than $1-\delta$. In particular, $\epsilon = c = a_{00}+a_{10}$ in the above example. (Here, we implicitly assume that $a_{00}+a_{10}$ is strictly greater than zero, i.e., $a_{00}+a_{10}>0$.)
\end{proof}

\subsection{Proof of Theorem~\ref{thm:general_lower}}
\label{proof:genera_lower}
\thmLower*

\begin{proof}[Proof of Theorem~\ref{thm:general_lower}]
Again, it suffices to show for the case where the confounder is binary. The extension to categorical confounder is straightforward as illustrated in the proof of Theorem~\ref{thm:lower_bound} in Appendix~\ref{proof:lower_bound}. Let $q_{yt} = P(Z=1|Y=y, T=t)$.
We will proceed by construction. Consider ${\bf q} = (q_{00}, q_{01}, \beta, \beta+\gamma)$ and ${\bf q}' = (q_{00}, q_{01}, \beta+ \gamma, \beta )$, for some small $\gamma$. Then
\begin{align*}
&\ATE_{\bf a}({\bf q}) = 
\frac{a_{11}(\beta+\gamma)}{a_{11}(\beta+\gamma) + a_{01}q_{01}}
(a_{00}q_{00} + a_{01}q_{01} + a_{10}\beta+ a_{11}(\beta+\gamma)) 
+ \frac{a_{11}(1-\beta-\gamma)}{a_{11}(1-\beta-\gamma) + a_{01}(1-q_{01})}
\\&(1- 
a_{00}q_{00} -
a_{01}q_{01} - a_{10}\beta 
- a_{11}(\beta+\gamma)) 
-
\frac{a_{10}\beta}{a_{10}\beta+a_{00}q_{00}}(a_{00}q_{00} + a_{01}q_{01} + a_{10}\beta+ a_{11}(\beta+\gamma))
-
\\& \frac{a_{10}(1-\beta)}{a_{10}(1-\beta)+a_{00}(1-q_{00})}(1- a_{00}q_{00} - a_{01}q_{01} - a_{10}\beta 
- a_{11}(\beta+\gamma)) ,
\end{align*}
and similarly, we have
\begin{align*}
&\ATE_{\bf a}({\bf q}') =
\frac{a_{11}\beta}{a_{11}\beta + a_{01}q_{01}}
(a_{00}q_{00} + a_{01}q_{01} + a_{10}(\beta+\gamma)+ a_{11}\beta) +
\frac{a_{11}(1-\beta)}{a_{11}(1-\beta) + a_{01}(1-q_{01})}
(1- a_{00}q_{00} -
\\&
a_{01}q_{01} - a_{10}(\beta+\gamma)
- a_{11}\beta 
) -
\frac{a_{10}(\beta+\gamma)}{a_{10}(\beta+\gamma)+a_{00}q_{00}}(a_{00}q_{00} + a_{01}q_{01} + a_{10}(\beta+\gamma)+ a_{11}\beta) - 
\\&
\frac{a_{10}(1-\beta-\gamma)}{a_{10}(1-\beta-\gamma)+a_{00}(1-q_{00})}(1- a_{00}q_{00} - a_{01}q_{01} - a_{10}(\beta+\gamma)
- a_{11}\beta
).
\end{align*}
Ignoring the $\gamma$ in the denominator, we have that 
\begin{align}
&\ATE_{\bf a}({\bf q})-\ATE_{\bf a}({\bf q}') =
(\frac{a_{11}}{a_{11}\beta+a_{01}q_{01}}+\frac{a_{10}}{a_{10}\beta+a_{00}q_{00}})(a_{00}q_{00}+a_{01}q_{01}+a_{10}\beta+a_{11}\beta)\gamma \nonumber
\\& 
-(\frac{a_{11}}{a_{11}(1-\beta)+a_{01}(1-q_{01})}+\frac{a_{10}}{a_{10}(1-\beta)+a_{00}(1-q_{00})})(1-a_{00}q_{00}-a_{01}q_{01}-a_{10}\beta-a_{11}\beta)\gamma \nonumber
\\& 
+\frac{a_{11}^2-a_{11}a_{10}}{a_{11}\beta+a_{01}q_{01}}\beta\gamma - \frac{a_{11}^2-a_{11}a_{10}}{a_{11}(1-\beta)+a_{01}(1-q_{01})}(1-\beta)\gamma \nonumber
\\&
+\frac{a_{10}^2-a_{11}a_{10}}{a_{10}\beta+a_{00}q_{00}}\beta\gamma - \frac{a_{10}^2-a_{11}a_{10}}{a_{10}(1-\beta)+a_{00}(1-q_{00})}(1-\beta)\gamma \nonumber
\\&
+ \frac{a_{11}^2}{a_{11}\beta+a_{01}q_{01}}\gamma^2 + \frac{a_{11}^2}{a_{11}(1-\beta)+a_{01}(1-q_{01})}\gamma^2
+ \frac{a_{10}^2}{a_{10}\beta+a_{00}q_{00}}\gamma^2 + \frac{a_{10}^2}{a_{10}(1-\beta)+a_{00}(1-q_{00})}\gamma^2 \label{eq:ATE_diff}
\end{align}

Similar to the proof above, let $B_1$ denote the Bernoulli random variable with parameter $\beta$, and let $B_2$ denote the Bernoulli random variable with parameter $\beta+\gamma$. Then, given $m$ deconfounded samples, we have $\mathrm{KL}(B_1^m\|B_2^m) \leq m\beta\ln(\frac{\beta}{\beta+\gamma}) + m(1-\beta)\ln(\frac{1-\beta}{1-\beta-\gamma})\leq m\ln(1+\frac{\gamma}{1-\beta-\gamma}) \leq m(\frac{\gamma}{1-\beta-\gamma} - \frac{\gamma^2}{2(1-\beta-\gamma)^2}) $. Thus, we have $m\sim \Omega(\frac{\ln(\delta^{-1})}{\gamma^2})$. From Equation (\ref{eq:ATE_diff}), we observe that $\epsilon = \|\ATE_{\bf a}({\bf q}) - \ATE_{\bf a}({\bf q}')\| \sim \Omega (\gamma)$. Combining above, we have $m\sim \Omega(\frac{\ln(\delta^{-1})}{\epsilon^2})$.

\end{proof}

\subsection{Proof of Theorems~\ref{thm_m1} and \ref{thm_m2}}\label{proof:thm_m2}

\thmNSP*
\thmUSPOWSP*

\begin{proof}[Proof of Theorems~\ref{thm_m1} and \ref{thm_m2}]
In these theorems, we derive the concentration of the $\widehat{\ATE}$ assuming infinite confounded data, and parametrize $p_{yt}^z$ by $p_{yt}^z = a_{yt}q_{yt}^z$.  Since under infinite confounded data, $a_{yt}$'s are known, and thus we only need to estimate the $q_{yt}^z$'s. 
The key difference between Theorem~\ref{thm_m2} and Theorem~\ref{thm_m0} is that now 
we define the random variables $X_{yt}^z$ to map the event $(Z=z|Y=y,T=t)$ to $\{0,1\}$. Thus, $X_{yt}^z$ is distributed according to Bernoulli$(q_{yt}^z)$. 
Thus, to decompose $\left|a_{1t}q_{1t}^z + a_{0t}q_{0t}^z  - a_{1t}\hat{q}_{1t}^z -  a_{0t}\hat{q}_{0t}^z \right|$,
we first show the following lemma:
\begin{lemma}\label{cor4}
Let $X_1, ..., X_{x_1 m}$ and $Y_1, ..., Y_{x_2 m}$ be independent random variables in [0,1]. Then for any $t>0$, we have 
$$P\left(\left|\alpha\frac{\sum_{i=1}^{x_1 m} X_i-\mathbb{E}\left[X_i\right]}{x_1 m} + \beta\frac{\sum_{j=1}^{x_2 m} Y_j-\mathbb{E}\left[Y_j\right]}{x_2 m}\right|
\geq \alpha t+\beta k\right)\leq 
2\exp\left(-\frac{2m (\alpha t+\beta k)^2}{\left(\frac{\alpha^2}{x_1}+\frac{\beta^2}{x_2}\right)}\right).$$
\end{lemma}
\begin{proof}[Proof of Lemma \ref{cor4}]
First observe that 
\begin{align*}
P & \left(\alpha\frac{\sum_{i=1}^{x_1 m} X_i-\mathbb{E}\left[X_i\right]}{x_1 m} + \beta\frac{\sum_{j=1}^{x_2 m} Y_j-\mathbb{E}\left[Y_j\right]}{x_2 m}\geq \alpha t+\beta k\right)
\\
&=P\bigg(\frac{\alpha}{x_1}\sum \limits_{i=1}^{x_1 m} (X_i-\mathbb{E}\left[X_i\right])+ \frac{\beta}{x_2}\sum \limits_{j=1}^{x_2 m} (Y_j-\mathbb{E}\left[Y_j\right])\geq m\alpha t+m\beta k\bigg).
\end{align*}
Now, let $Z_i=\frac{\alpha}{x_1}X_i$ if $i\in [1,x_1 m]$, and $Z_i=\frac{\beta}{x_2}Y_i$ if $i\in [x_1 m +1, (x_1 + x_2)m]$. Then applying Theorem~\ref{hoeffding}, we have
\begin{align*}
   P\left(\left|\sum_{i=1}^{(x_1+x_2)m} \left(Z_i - \mathbb{E}[Z_i]\right)\right|\geq m\alpha t + m\beta k \right) 
   & \leq 2\exp\left(-\frac{2m^2(\alpha t + \beta k)^2}{\sum_{i=1}^{(x_1+x_2)m} (M_i-m_i)^2}\right) \\
& = 2\exp\left(-\frac{2m(\alpha t + \beta k)^2}{ \frac{\alpha^2}{x_1}+\frac{\beta^2}{x_2}}\right). 
\end{align*}
\end{proof}
As defined in Section~\ref{sec:methods}, let $x_{yt}$ denote the percentage data we sample from the group $yt$.

Recall that from the proof of Theorem~\ref{thm_m0}, we have
\begin{align*}
&P\left(|\widehat{\ATE}-\ATE|\geq \epsilon \right)
\leq 
\sum _{t,z}
P\left( \left|
p_{1t}^z - \hat{p}_{1t}^z
\right| \geq \frac{\sum_y p_{yt}^z}{5k}\epsilon \right)
+ P\left(\left|p_{1t}^z + p_{0t}^z  - \hat{p}_{1t}^z -  \hat{p}_{0t}^z \right| \geq \frac{\sum_y p_{yt}^z}{5k}\epsilon \right)    
\\&
= \sum _{t,z}
P\left(\left|
a_{1t}q_{1t}^z - a_{1t}\hat{q}_{1t}^z
\right| \geq \frac{\sum_y a_{yt}q_{yt}^z}{5k}\epsilon \right)
+ P\left(\left|a_{1t}q_{1t}^z + a_{0t}q_{0t}^z  - a_{1t}\hat{q}_{1t}^z -  a_{0t}\hat{q}_{0t}^z \right| \geq \frac{\sum_y a_{yt}q_{yt}^z}{5k}\epsilon\right)
\\&
= \sum _{t,z}
P\left(\left|
q_{1t}^z - \hat{q}_{1t}^z
\right| \geq \frac{\sum_y a_{yt}q_{yt}^z}{5k a_{1t}}\epsilon \right)
+ P\left(\left|a_{1t}q_{1t}^z + a_{0t}q_{0t}^z  - a_{1t}\hat{q}_{1t}^z -  a_{0t}\hat{q}_{0t}^z \right| \geq \frac{\sum_y a_{yt}q_{yt}^z}{5k}\epsilon\right)
\\& 
\leq 4k\max_{t,z}\left( 
2\exp\left(
-2x_{1t}m\frac{\left(\sum_y a_{yt}q_{yt}^z\right)^2\epsilon^2}{25k^2a_{1t}^2}
\right),
2\exp\left(-2m \frac{\left(\sum_y a_{yt}q_{yt}^z\right)^2\epsilon^2}{25k^2\sum_y \frac{a_{yt}^2}{x_{yt}}} \right)
\right) \\ & \leq \delta,
\end{align*}
where the second to last line follows from applying Lemma~\ref{cor4} to the second half of the line above it.
Solving the equation above, we have
\begin{align*}
m &\geq  \frac{12.5k^2\ln(\frac{8k}{\delta})}{\epsilon^2} \max_{t,z} 
\left( \frac{a_{1t}^2/x_{1t}}{\left(\sum_y a_{yt}q_{yt}^z\right)^2}, 
\frac{\sum_y \left(a_{yt}^2/x_{yt}\right)}{\left(\sum_y a_{yt}q_{yt}^z\right)^2}
\right)
=\frac{12.5k^2\ln(\frac{8k}{\delta})}{\epsilon^2} \max_{t,z} \frac{\sum_y \left(a_{yt}^2/x_{yt}\right)}{\left(\sum_y a_{yt}q_{yt}^z\right)^2}.
\end{align*}
The last equality is because $a_2^2/x_2, a_1^2/x_1>0$. 
Under NSP, $x_{yt} = a_{yt}$. Thus, we have
$$
m_{\mathrm{nsp}}:= \frac{12.5k^2\ln(\frac{8k}{\delta})}{\epsilon^2} \max_{t,z} \frac{\sum_y a_{yt}}{\left(\sum_y a_{yt}q_{yt}^z\right)^2}.
$$
Similarly, under USP, $x_{yt}=\frac{1}{4}$, and we have
$$
m_{\mathrm{usp}}:= \frac{12.5k^2\ln(\frac{8k}{\delta})}{\epsilon^2} \max_{t,z} \frac{\sum_y 4a_{yt}^2}{\left(\sum_y a_{yt}q_{yt}^z\right)^2}.
$$
Lastly, under OWSP, $x_{yt} = \frac{a_{yt}}{2\sum_y a_{yt}}$, and we have 
$$
m_{\mathrm{owsp}}:= \frac{12.5k^2\ln(\frac{8k}{\delta})}{\epsilon^2} \max_{t,z} \frac{2(\sum_y a_{yt})^2}{\left(\sum_y a_{yt}q_{yt}^z\right)^2}.
$$

\end{proof}

\subsection{Proof of Theorem~\ref{cor}}\label{proof:cor}

\corUpper*

\begin{proof}[Proof of Theorem~\ref{cor}]
We proceed by construction. 
For simplicity, we illustrate the correctness of Theorem~\ref{cor}
for binary confounders. The extension to the multi-valued confounder is straightforward and will be demonstrated in the proof of Theorem~\ref{thm:lower_bound}.

Consider the following example: 
$ a_{01} = a_{10} = a_{11} = \eta$,  $a_{00} = 1 - 3\eta$,
and consider the following pair of $\bf q$'s: ${\bf q} = (\beta, \beta, \beta, c\beta)$ and ${\bf q}' = (\beta, \beta, \beta, \beta)$, where $c\leq \frac{1-\beta}{\beta}$ is some constant. 
Here, one of the $\bf q$ and ${\bf q}'$ represents the true ATE, and the other represents the estimated ATE using the best estimator. Without loss of generality, we assume that we have already identified three components of the true conditional distribution. (In general, we can always construct an instance by modifying the values of $a_{01}$ and $a_{10}$ so that the majority error is induced by estimation error on $q_{11}$.)
Then, we have $\ATE_{\bf a}({\bf q}) = \frac{c\beta}{1+c} + \frac{(1-c\beta)(1-\beta)}{2-c\beta-\beta} - \frac{\eta}{1-2\eta}$, and $\ATE_{\bf a}({\bf q}') = \frac{1}{2} - \frac{\eta}{1-2\eta}$. Thus,  
$\Delta \ATE := |\ATE_{\bf a}({\bf q}) - \ATE_{\bf a}({\bf q}') |$:
\begin{align*}
    \Delta \ATE
    &=\left|
    \frac{1}{2} -
    \frac{c\beta}{c+1} - \frac{(1-c\beta)(1-\beta)}{2-c\beta-\beta}
    \right|.
\end{align*}
Note that when $c = \frac{1-\beta}{\beta}$, $\Delta \ATE = 0.5 - 2\beta(1-\beta)\approx 0.5$. Thus, for any $\epsilon \in [0,0.5-2\beta(1-\beta)]$, there exists some $c$ such that
$\epsilon = \Delta\ATE$. Then, for any $\delta$,
let $\mu$ denote the minimum expected number of
samples that we need to
distinguish $\bf q$ from $\bf q'$ under the best estimator. Then under NSP, the minimum number of samples that we need under the best estimator equals to $\mu_{\mathrm{nsp}} := \mu/\eta$, and under OWSP, the minimum number of samples that we need under the best estimator equals to $\mu_{\mathrm{oswp}} = 4\mu$. (Note that ${\bf x}_{yt} = (\frac{1-3\eta}{2(1-2\eta)}, \frac{1}{4}, \frac{\eta}{2(1-2\eta)}, \frac{1}{4})$ under OWSP in this example.)
Thus, $\mu_{\mathrm{owsp}}/\mu_{\mathrm{nsp}} = 4\eta$. Since in this example, $\eta$ is at most $1/4$, $\mu_{\mathrm{owsp}}/\mu_{\mathrm{nsp}} \leq 1$ and can be arbitrarily close to $0$ as $\eta\rightarrow 0$.  
(Intuitively, the first statement is true because when $\sum_t a_{0t}\ll \sum_t a_{1t}$ and $a_{00}\approx a_{01}$, it is equally important to estimate $q_{0t}^z$'s and $q_{1t}^z$'s according to the ATE expression. However, under this setup, the number of samples allocated to groups $(0,t)$'s decreases as $a_{0,t}$'s approach to 0 under NSP, while under OWSP, half of the deconfounded samples are always dedicated to estimate the $q_{0t}^z$'s.)

Next, we show the last sentence in Theorem~\ref{cor} is true. For any fixed $\epsilon, \delta<1$, let $\mu_{\mathrm{nsp}}$ be the minimum expected number of samples needed to achieve $P(|\widehat\ATE-\ATE|\geq\epsilon)<\delta$ under natural selection policy for the best estimator, then when $w_{\mathrm{owsp}}:= 2\mu_{\mathrm{nsp}}\max_t \sum_{y} a_{yt}$ also achieves $P(|\widehat\ATE-\ATE|\geq\epsilon)<\delta$ under the outcome-weighted selection policy. The reason is that when using $w_{\mathrm{owsp}}$
number of deconfounded samples, the number of deconfounded data allocated to each $yt$ group is at least as much as those under the natural selection policy. Thus, we have $\mu_{\mathrm{owsp}}\leq w_{\mathrm{owsp}} \leq 2\mu_{\mathrm{nsp}}$, where the last inequality is because $\max_{t}\sum_y a_{yt}< 1$.

\end{proof}

\subsection{Proof of Theorem~\ref{thm:lower_bound}}\label{proof:lower_bound}
\thmLowerPolicies* 

\begin{proof}
Consider ${\bf q} = (q_{00}^z, q_{01}^z, q_{10}^z, q_{11}^z)$ where
$q_{01}^1 = \beta$,  $q_{11}^1 = \beta+\gamma$, 
and $q_{11}^z = q_{01}^z - \gamma/(k-1) $ for $z=2,...,k$, with
$\sum_z q_{01}^z = \sum_{z} q_{11}^z =1$.
We assume that $q_{11}^z, q_{01}^z \in[\beta, 1-\beta]$ for some suitable $\beta$ and $\gamma$ for all values of $Z$.
Similarly, we consider the ${\bf q}'$ where the entries of $q_{01}^z$ and $q_{11}^z$ are flipped, i.e., ${\bf q}'= (q_{00}^z,  q_{11}^z, q_{10}^z, q_{01}^z )$, for some small $\gamma$, where the $q_{yt}^z$'s are defined above. Then,
\begin{align*}
&\ATE_{\bf a}({\bf q}) =
\sum_z\left(\left(\frac{a_{11}q_{11}^z}{\sum_y a_{y1}q_{y1}^z} - \frac{a_{10}q_{10}^z}{\sum_y a_{y0}q_{y0}^z}\right)
\sum_{y,t}a_{yt}q_{yt}^z\right)
\\&
=
\frac{a_{11}(\beta+\gamma)}{a_{11}(\beta+\gamma) + a_{01}\beta}
(a_{00}q_{00}^1 + a_{01}\beta + a_{10}q_{10}^1+ a_{11}(\beta+\gamma)) -
\frac{a_{10}q_{10}^1}{a_{10}q_{10}^1+a_{00}q_{00}^1}(a_{00}q_{00}^1 + a_{01}\beta +
a_{10}q_{10}^1 +
\\& 
a_{11}(\beta+\gamma))
+ 
\sum_{z=2}^k \frac{a_{11}\left(q_{01}^z - \frac{\gamma}{k-1}\right)}{a_{11}\left(q_{01}^z - \frac{\gamma}{k-1}\right) + a_{01}q_{01}^z}
\left(a_{00}q_{00}^z + a_{01}q_{01}^z +  a_{10}q_{10}^z
+ a_{11}\left(q_{01}^z - \frac{\gamma}{k-1}\right)\right)
-
\\& \sum_{z=2}^k  \frac{a_{10}q_{10}^z}{a_{10}q_{10}^z+a_{00}q_{00}^z}\left( a_{00}q_{00}^z + a_{01}q_{01}^z + a_{10}q_{10}^z
+ a_{11}\left(q_{01}^z - \frac{\gamma}{k-1}\right)\right),
\end{align*}
and similarly, we have
\begin{align*}
&\ATE_{\bf a}({\bf q}') =
\frac{a_{11}\beta}{a_{11}\beta + a_{01}(\beta+\gamma)}
(a_{00}q_{00}^1 + a_{01}(\beta+\gamma) + a_{10}q_{10}^1+ a_{11}\beta) - \frac{a_{10}q_{10}^1}{a_{10}q_{10}^1+a_{00}q_{00}^1}(a_{00}q_{00}^1 + a_{01}(\beta+\gamma) +
\\&
a_{10}q_{10}^1+ a_{11}\beta) +
\sum_{z=2}^k \frac{a_{11}q_{01}^z}{a_{11}q_{01}^z + a_{01}\left(q_{01}^z-\frac{\gamma}{k-1}\right)}\left(a_{00}q_{00}^z + a_{01}\left(q_{01}^z-\frac{\gamma}{k-1}\right) +  a_{10}q_{10}^z
+ a_{11}q_{01}^z\right)
-
\\& \sum_{z=2}^k  \frac{a_{10}q_{10}^z}{a_{10}q_{10}^z+a_{00}q_{00}^z}\left( a_{00}q_{00}^z + a_{01}\left(q_{01}^z-\frac{\gamma}{k-1}\right) + a_{10}q_{10}^z
+ a_{11}q_{01}^z\right)
\end{align*}
Ignoring the $\gamma$ in the denominator, we have that 
\begin{align}
&\ATE_{\bf a}({\bf q})-\ATE_{\bf a}({\bf q}') 
\approx
\frac{a_{11}}{a_{11}\beta+a_{01}\beta}
(a_{00}q_{00}^1+a_{01}\beta+a_{10}q_{10}^1+a_{11}\beta)\gamma 
+\frac{a_{10}q_{10}^1(a_{01}-a_{11})}{a_{10}q_{10}^1+a_{00}q_{00}^1}\gamma 
\nonumber
\\& 
-\left(\sum_{z=2}^k\left(\frac{a_{11}/{k-1}}{a_{11}q_{01}^z+a_{01}q_{01}^z}  (a_{00}q_{00}^z+a_{10}q_{10}^z+(a_{01}+a_{11})q_{10}^z)\right)\right)\gamma 
- \sum_{z=2}^k\frac{a_{10}q_{10}^z(a_{01}-a_{11})}{a_{10}q_{10}^z+a_{00}q_{00}^z}\frac{1}{k-1}\gamma
\nonumber
\\&
+ \frac{a_{11}^2}{a_{11}\beta+a_{01}\beta}\gamma^2 +
\sum_{z=2}^k\frac{a_{11}^2}{a_{11}q_{01}^z+a_{01}q_{01}^z}\frac{\gamma^2}{(k-1)^2} \nonumber
\\&
=\frac{a_{11}}{a_{11}\beta+a_{01}\beta}
(a_{00}q_{00}^1+a_{10}q_{10}^1)\gamma 
+\frac{a_{10}q_{10}^1(a_{01}-a_{11})}{a_{10}q_{10}^1+a_{00}q_{00}^1}\gamma 
-\frac{a_{11}}{k-1} \sum_{z=2}^k\left(\frac{a_{00}q_{00}^z+a_{10}q_{10}^z}{a_{11}q_{01}^z+a_{01}q_{01}^z}  \right)\gamma
\nonumber
\\& 
- \frac{1}{k-1}\sum_{z=2}^k\frac{a_{10}q_{10}^z(a_{01}-a_{11})}{a_{10}q_{10}^z+a_{00}q_{00}^z}\gamma
+ \frac{a_{11}^2}{a_{11}\beta+a_{01}\beta}\gamma^2 +
\sum_{z=2}^k\frac{a_{11}^2}{a_{11}q_{01}^z+a_{01}q_{01}^z}\frac{\gamma^2}{(k-1)^2}
\label{eq:ATE_diff2}
\end{align}

Since the second order terms in $\gamma$ is dominated by the first order terms in $\gamma$, thus to find the highest lower bound for sample complexity in this instance is to find the largest coefficient in front of $\gamma$.

Assuming that $\beta \ll k$ and $k\beta <1$, then the maximum of Equation (\ref{eq:ATE_diff2}) is achieved when $q_{00}^z = q_{10}^z =\beta$ , $q_{00}^1 = q_{10}^1 = 1-k\beta$, and $q_{01}^z = (1-\beta)/(k-1)$, and the coefficient in front of $\gamma$ is
$$\frac{a_{11}}{a_{11}+a_{01}}(a_{00}+a_{10})(\frac{1}{\beta}-\frac{k-\beta}{1-\beta})\approx \frac{a_{11}}{a_{11}+a_{01}}(a_{00}+a_{10})\left(\frac{1}{\beta}-k\right).$$

Similar to the proof of Theorem~\ref{thm:general_lower}, 
we have $m\sim \Omega(\frac{\ln(\delta^{-1})}{\gamma^2})$. From Equation (\ref{eq:ATE_diff}), we observe that $\epsilon = \|\ATE_{\bf a}({\bf q}) - \ATE_{\bf a}({\bf q}')\| \sim \Omega (\gamma)$. Combining above, we have $m\sim \Omega(\frac{\ln(\delta^{-1})}{\epsilon^2})$.
In the case above, $\epsilon\approx \frac{a_{11}}{a_{11}+a_{01}}(a_{00}+a_{10})\frac{1}{\beta}\gamma$, thus, the number of deconfounded samples needed is approximately $$ m \propto \frac{\ln(\delta^{-1})a_{11}^2(a_{00}+a_{10})^2 }{\epsilon^2  (a_{11}+a_{01})^2}\left(\frac{1}{\beta} - k\right)^2.$$
Let $C_1 \propto (k\beta - 1)^2\ln(\delta^{-1})\epsilon^{-2}$. Then $m\sim \Omega\left(\frac{C_1}{\beta^2}\frac{a_{11}^2(a_{00}+a_{10}^2)}{(a_{11}+a_{01})^2}\right)$.

If we flip the values of $q_{01}^z$ and $q_{11}^z$ with the values of $q_{00}^z$ and $q_{10}^z$ in both $\bf q$ and ${\bf q}'$, then we have $m\sim \frac{C_1}{\beta^2} \frac{a_{10}^2(a_{01}+a_{11})^2}{ (a_{10}+a_{00})^2}$. Note that this is because that the estimation error on ATE and $1-\ATE$ is symmetric.
In addition, under natural selection policy, we need at least $\frac{m}{a_{11}}$ samples; uniform selection policy, we need at least $4m$ deconfounded samples; under outcome-weighted selection policy, we need at least $2\frac{a_{11}+a_{01}}{a_{11}}m$ deconfounded samples.
Combining all of the above, we obtained Theorem~\ref{thm:lower_bound}.

\end{proof}


\section{Finite Confounded Data}\label{app:finite}


In this case,  \emph{deconfounding} reveals the value of $Z$ for one (initially confounded) sample, 
and thus we gain no additional information about $\mathcal{P}_{Y,T}$. 
Thus,
these $n$ confounded data provide us with an \emph{estimate} 
of the confounded distribution, $ \hat{P}_{Y,T}(y,t)$, which we denote $\hat{a}_{yt}$, and thus provide us an estimated OWSP.
Similarly, we estimate $\hat{a}_{yt}$ using the MLE from the confounded data.
To check the robustness of OWSP, 
we extend our analysis 
to handle
finite confounded data.
%
%
%
%
%
%
%
With $x_{yt}$ defined as in Section~\ref{sec:infinite}, 
we can derive a theorem analogous to Theorems \ref{thm_m0}-\ref{thm_m2}:
\begin{restatable}{theorem}{thmFinite}{\red (Upper Bound)}
\label{thm5}
Given $n$ confounded and $m$ deconfounded samples, with $n\geq m$,
$P(|\ATE-\widehat{\ATE}|\geq\epsilon)\leq \delta$ is satisfied when
\begin{align}
\min_{y, t,z} 
\frac{\left(\sum_y a_{yt}q_{yt}^z\right)^2}{\frac{1}{x_{yt}m} + 
\frac{(q_{yt}^z)^2}{n}}
&=
\min \limits_{y,t,z}
\left(
{\frac{P_{T,Z}(t,z)^2}{\frac{1}{x_{yt}m}+\frac{(q_{yt}^z)^2}{n}}}
\right)
\geq
4C.\label{eq:finite}
\end{align}
\vspace{-20pt}
\end{restatable}
The proof of Theorem~\ref{thm5} (Appendix~\ref{proof:thm5}) 
requires a bound we derive (Appendix, Lemma~\ref{thm3}) for the product of two independent random variables.
A few results follow from Theorem~\ref{thm5}. First, a quick calculation
shows that when $m$ is held constant, 
$P(|\ATE-\widehat{\ATE}|\geq\epsilon)$ remains positive as $n\rightarrow \infty$. 
This means that for a certain combinations of $\epsilon, \delta, n$, there does not necessarily exist
a sufficiently large $m$ s.t. $P(|\ATE-\widehat{\ATE}|\geq\epsilon)\leq \delta$ can be satisfied.
However, when there exists such an $m$, 
then 
\vspace{-10pt}
$$
m\geq \max_{y,t,z} {x_{yt}^{-1}\left(\frac{P_{T,Z}(t,z)^2}
{4C}
- \frac{(q_{yt}^z)^2}{n}\right)^{-1}}.
\vspace{-5pt}
$$

Although Theorem~\ref{thm5} does not recover Theorems~\ref{thm_m1} and \ref{thm_m2} exactly when $n\rightarrow\infty$,\footnote{
We could apply Lemma~\ref{lemma1} (Appendix~\ref{app_proofs}) 
to obtain a bound that recovers Theorems~\ref{thm_m1} and \ref{thm_m2} exactly as $n\rightarrow\infty$. 
However, this method does not give us sufficient insights
into the comparative performance of our sampling policies.}
it provides us with insights 
into relative performance of our sampling policies.
%
%
Theorem~\ref{thm5} implies that when $n \gg (q_{yt}^z)^2 x_{yt} m \; \forall y,t,$
the majority of the estimation error comes from not deconfounding enough data. 
This is because when the number of confounded data that we have
is more than $\Omega(m)$, the error on the ATE
in Equation (\ref{eq:finite}) is dominated by fact 
that we have not deconfounded enough data.
To put it another way, for a given $m$, having $n = \Omega(m)$ confounded samples is sufficient.

\subsection{Proof of Theorem~\ref{thm5}}\label{proof:thm5}
\thmFinite*

\begin{proof}[Proof of Theorem~\ref{thm5}]
In this theorem, we derive the concentration for the $\widehat{\ATE}$ under finite confounded data. 
The difference between Theorem~\ref{thm_m2} and Theorem~\ref{thm5} is that now we need to estimate $a_{yt}$ in addition to $q_{yt}^z$. Thus, 
to decompose $|a_{yt}q_{yt}^z - \hat{a}_{yt}\hat{q}_{yt}^z|$, we first derive Lemma~\ref{thm3}.
\subsubsection{Lemma~\ref{thm3}}\label{proof:thm3}
\begin{lemma}[Sample complexity for two independent r.v.s with two independent sampling processes]\label{thm3}
Let $X_1,..., X_n$ and $Y_1, ..., Y_m$ be two sequences of Bernoulli random variables independently drawn from distribution $p_1$ and $p_2$, respectively. Let $S_X= \sum \limits_{i=1}^n X_i, S_Y =\sum \limits_{i=1}^m Y_i$. Then, 
$$P\Big(\Big|S_X S_Y - \mathbb{E}\left[S_X\right]\mathbb{E}\left[S_Y\right]\Big|\geq nmt\Big)\leq 2\exp\left(\frac{-2t^2}{\frac{1}{m}+\frac{p_2^2}{n}}\right).$$
\end{lemma}

\begin{proof}[Proof of Lemma~\ref{thm3}]
The proof follows the proof of Hoeffding's inequality:
\begin{align*}
&P\Big(S_X S_Y - \mathbb{E}[S_X]\mathbb{E}[S_Y]\geq nmt\Big)= P\big(\exp(a S_X S_Y - a \mathbb{E}[S_X]\mathbb{E}[S_Y]))\geq \exp(anmt)\big)
\numberthis \label{eq:pMI}
\\ &
\leq \exp(-anmt)\mathbb{E}\left[\exp(a S_X S_Y - a \mathbb{E}[S_X]\mathbb{E}[S_Y]))\right] \numberthis \label{eq:MI}
, \quad \quad \,\, \text{  (because of Markov's inequality)}
\\&
=\exp(-anmt)\mathbb{E}\left[\exp(a S_X (S_Y-\mathbb{E}[S_Y])+ a \mathbb{E}[S_Y](S_X- \mathbb{E}[S_X])\right]
\\&
\leq \exp(-anmt)\mathbb{E}\left[\exp(a\max(S_X) (S_Y-\mathbb{E}[S_Y]) + a \mathbb{E}[S_Y](S_X- \mathbb{E}[S_X]))\right] 
\numberthis \label{eq:sx}
 \text{ (because } S_X\geq 0\text{)}
\\&
= \exp(-anmt)\mathbb{E}\left[\exp(an (S_Y-\mathbb{E}[S_Y])+ a\mathbb{E}[S_Y](S_X- \mathbb{E}[S_X]))\right]
\\ &
= \exp\left(-anmt\right)\mathbb{E}\left[\exp\left(a n (S_Y-\mathbb{E}[S_Y])\right)\right]\mathbb{E}\left[\exp(a \mathbb{E}[S_Y](S_X- \mathbb{E}[S_X]))\right]
\numberthis \label{eq:indp}
\quad \text{ (because} X \indep Y \text{)}
\\&
= \exp(-anmt)\prod_{i=1}^{m}\prod_{j=1}^{n} \mathbb{E}\left[\exp(an(Y_i-\mathbb{E}[Y_i]))\right]
\mathbb{E}\left[\exp(a\mathbb{E}[S_Y](X_j-\mathbb{E}[X_j]))\right]
\\&
\leq 
\exp(-anmt)\prod_{i=1}^{m}\exp\left(\frac{a^2}{8}n^2\right)\prod_{j=1}^{n}\exp\left(\frac{a^2}{8}\mathbb{E}[S_Y]^2\right)
\numberthis \label{eq:hoeffd}
\\&
= \exp\left(-anmt+\frac{a^2}{8}mn^2+\frac{a^2}{8}nm^2p_2^2\right) 
\numberthis \label{eq:min}
\, \text{ (because the minimum is achieved at $a=\frac{4t}{n+mp_2^2}$)}
\\&
\leq 
\exp\left(-\frac{2mnt^2}{n+mp_2^2}\right)
=\exp\left(-\frac{2t^2}{\frac{1}{m}+\frac{p_2^2}{n}}\right).
\end{align*}
Line (\ref{eq:hoeffd}) is because 
$Y_i-\mathbb{E}[Y_i]\in \{-\mathbb{E}[Y_i],1-\mathbb{E}[Y_i])$, and thus $n(Y_i-\mathbb{E}(Y_i))\in [-n\mathbb{E}[Y_i], n(1-\mathbb{E}[Y_i])]$.
Furthermore, $\mathbb{E}[S_Y](X_i-\mathbb{E}[X_i])\in (-\mathbb{E}[X]\mathbb{E}[S_Y],(1-\mathbb{E}[X])\mathbb{E}[S_Y])$. Finally, applying Hoeffding's Lemma (Lemma~\ref{lemma:hoeffding}), we obtain line (\ref{eq:hoeffd}). 
\end{proof}

Now we are ready to prove Theorem~\ref{thm5}.
\subsubsection{Proof of Theorem~\ref{thm5}}
In this theorem, we assume that the number of confounded data is finite. Thus, instead of $a_{yt}$, we have estimates of them, namely $\hat{a}_{yt}$. 
Let $n_{yt}$ denote the number of samples in the confounded data such that $(Y=y, T=t)$. Let $m_{yt}^z$ be the number of samples in the deconfounded data such that $(Y=y,T=t,Z=z)$. Furthermore, let $n=\sum_{y,t} n_{yt}, m=\sum_{y,t,z} m_{yt}^z$. Then, under our setup, we estimate $a_{yt}$ and $q_{yt}^z$ as follows: 
$$\hat{a}_{yt} = \frac{n_{yt}}{n}, \; \text{and}\;
\hat{q}_{yt}^z = \frac{m_{yt}^z}{\sum_z m_{yt}^z}.$$

Thus, following the proof of Theorem~\ref{thm_m0}, we have
\begin{align*}
&P\left(|\widehat{\ATE}-\ATE|<\epsilon \right)
\geq 
P\left(\bigcap_{t,z} \left\{\left|
p_{1t}^z - \hat{p}_{1t}^z
\right| < \frac{\sum_y p_{yt}^z}{5k}\epsilon \right\} \bigcap_{t,z}
\left\{\left|p_{1t}^z + p_{0t}^z  - \hat{p}_{1t}^z -  \hat{p}_{0t}^z \right|<\frac{\sum_y p_{yt}^z}{5k}\epsilon
\right\}\right)
\\&
= P\left(\bigcap_{t,z} \left\{\left|
a_{1t}q_{1t}^z - \hat{a}_{1t}\hat{q}_{1t}^z
\right| < \frac{\sum_y a_{yt}q_{yt}^z}{5k}\epsilon \right\} \bigcap_{t,z}
\left\{\left|a_{1t}q_{1t}^z + a_{0t}q_{0t}^z  - \hat{a}_{1t}\hat{q}_{1t}^z -  \hat{a}_{0t}\hat{q}_{0t}^z \right|<\frac{\sum_y a_{yt}q_{yt}^z}{5k}\epsilon
\right\}\right).
\end{align*}
Notice that $\left|a_{1t}q_{1t}^z + a_{0t}q_{0t}^z  - \hat{a}_{1t}\hat{q}_{1t}^z -  \hat{a}_{0t}\hat{q}_{0t}^z \right|<\frac{\sum_y a_{yt}q_{yt}^z}{5k}\epsilon$ is satisfied when both $$\left|
a_{1t}q_{1t}^z - \hat{a}_{1t}\hat{q}_{1t}^z
\right| < \frac{\sum_y a_{yt}q_{yt}^z}{10k}\epsilon, \;\text{and}\; \left|
a_{0t}q_{0t}^z - \hat{a}_{0t}\hat{q}_{0t}^z
\right| < \frac{\sum_y a_{yt}q_{yt}^z}{10k}\epsilon.$$ We have:
\begin{align*}
P\left(|\widehat{\ATE}-\ATE|<\epsilon \right)
&\geq 
P\left(\bigcap_{t,z} \left\{\left|
a_{1t}q_{1t}^z - \hat{a}_{1t}\hat{q}_{1t}^z
\right| < \frac{\sum_y a_{yt}q_{yt}^z}{10k}\epsilon \right\} \bigcap_{t,z}
\left\{\left|
a_{0t}q_{0t}^z - \hat{a}_{0t}\hat{q}_{0t}^z
\right| < \frac{\sum_y a_{yt}q_{yt}^z}{10k}\epsilon
\right\}\right)
\\&
= P\left(\bigcap_{y,t,z} \left\{\left|
a_{yt}q_{yt}^z - \hat{a}_{yt}\hat{q}_{yt}^z
\right| < \frac{\sum_y a_{yt}q_{yt}^z}{10k}\epsilon \right\} \right).
\end{align*}
Lemma~\ref{thm3} suggests that 
$$P(|a_{yt}q_{yt}^z - \hat{a}_{yt}\hat{q}_{yt}^z|\geq t)\leq 2\exp\left(-\frac{2t^2}{\frac{1}{x_{yt}m}+\frac{(q_{yt}^z)^2}{n}}\right).$$
Thus, applying a union bound and Lemma~\ref{thm3}, we have
\begin{align*}
P\left(|\widehat{\ATE}-\ATE|\geq \epsilon \right)
&\leq 
\sum _{y,t,z}
P\left( \left|
a_{yt}q_{yt}^z - \hat{a}_{yt}\hat{q}_{yt}^z
\right| < \frac{\sum_y a_{yt}q_{yt}^z}{10k}\epsilon  \right)
\\&  
\leq 8k\max_{y, t,z} 
\exp\left(
-2\frac{\left(\sum_y a_{yt}q_{yt}^z\right)^2\epsilon^2}{(\frac{1}{x_{yt}m} + 
\frac{(q_{yt}^z)^2}{n})100k^2}
\right) \\&\leq \delta.
\end{align*}

Simplifying the equations above, we have
$$
\min_{y, t,z} 
\frac{\left(\sum_y a_{yt}q_{yt}^z\right)^2}{(\frac{1}{x_{yt}m} + 
\frac{(q_{yt}^z)^2}{n})}
\geq \frac{50k^2\ln\left(\frac{8k}{\delta}\right)}{\epsilon^2}.
$$

\end{proof}

\section{Corresponding Stories}\label{stories}
In this section, we will provide an example for each selection method such that this particular sampling performs the worst when compared with the other two methods. For the purpose of illustration, we consider binary confounder throughout this section. To ease notation, let $q_{yt}$ denote $q_{yt}^1$.

\paragraph{A Scenario in Which NSP Performs the Worst}
A drug repositioning start-up discovered that drug $T$ can potentially cure a disease $\gamma$. which has no known drug cure and goes away without treatments once a while. 
Since drug $T$ is commonly used to treat another disease $\eta$, the majority patients who has disease $\gamma$ do not receive any treatment. 
Among the ones who received drug $T$, the start-up discovered that the health outcomes of the majority of  patients have improved. 
The start-up proposes to bring drug $T$ to an observational study to verify whether drug $T$ could treat disease $\gamma$ while not controlling for patient's treatment adherence levels.
As in most cases, patient's treatment adherence levels could influence doctors' decision of whether to prescribe drug $T$ and whether the treatment for disease $\gamma$ will be successful.
Translating this scenario into our notations, we have $a_{01} = \epsilon_1$, $a_{10}=\epsilon_2$, $a_{11}=\epsilon_3$, and $a_{00} = 1-\sum_{i=1}^3\epsilon_i$, say $\mathbf{a}=(0.9,0.02,0.01,0.07)$. Now, imagine in the clinical trial, the patients are given a drug case containing drug $T$ such that the drug case automatically records the frequency that the patient takes the drug. Somehow we know a priori that the patients who do not have health improvement have on average poor treatment adherence, e.g., $q_{00}=0.9, q_{01}=0.7$; furthermore, those who have health improvement on average have good treatment adherence, e.g., $q_{10}=0.01, q_{11}=0.3$. Deconfounding according to NSP, i.e., $\mathbf{x}=(a_{00}, a_{01}, a_{10}, a_{11})$, in this case, will select most samples from the group $(Y=0,T=0)$. 
Since the ATE 
depends on the estimation that relies on both $T=0$, and $T=1$,
one would expect that NSP and OWSP will outperform NSP. 
The left column in Figure~\ref{examples} confirms this hypothesis.

\paragraph{A Scenario in Which USP Performs the Worst}

A group biostatisticians discovered that mutations on gene $T$ is likely to cause cancer $Y$ in patients with a particular type of heart disease. In particular, they discovered that among the those heart disease patients, 79\% of patients have neither mutation on $T$ nor cancer $Y$; 18\% patients have both mutation on $T$ and cancer $Y$. In other words, $a_{00}=0.79, a_{11}=0.18.$ Furthermore, we have $a_{01}=0.01,a_{10}=0.02$. This group of biostatisticians want to run a small experiment to confirm whether gene $T$ causes cancer $Y$. In particular, they are interested in knowing whether those patients also have mutations on gene $Z$, which is also suspected by the same group of biostatisticians to cause cancer $Y$. Somehow, we know a priori that $q_{00}=0.5, q_{01}=0.01, q_{10}=0.05, q_{11}=0.5$. From the calculation of the ATE, it is not difficult to observe that the error on the ATE is dominated by the estimation errors on $q_{00}, q_{11}$. Thus, we should sample more from the groups $(Y=0,T=0)$ and $(Y=1,T=1)$. 

\paragraph{A Scenario in Which OWSP Performs the Worst}

A team wants to reposition drug $T$ to cure diabetes. Drug $T$ has been used to treat a common comorbid condition of diabetes that appears in 31\% of the diabetic patient population. Among those patients who receive drug $T$, about 97\% has improved health, that is $a_{01}=0.01$ and $a_{11}=0.3$. Among the patients who have never received drug $T$, about 70\% have no health improvement, that is $a_{00}=0.5$, and $a_{10}=0.19$. Let $q_{00}=0.05, q_{01}=0.5, q_{10}=0.055, \text{ and } q_{11}=0.4$. In the ATE, it is easy to observe that $\frac{a_{11}q_{11}}{a_{11}q_{11}+a_{01}q_{01}}$ and $\frac{a_{11}(1-q_{11})}{a_{11}(1-q_{11})+a_{01}(1-q_{01})}$ are both dominated by 1 regardless of the estimates of $q_{11}$ and $q_{01}$. In this case, USP outperforms OWSP and NSP when the sample size is larger than $200$. On the other hand, the bottom figure in the third column of Figure~\ref{examples} shows that, when averaged over all possible values of $\mathbf{q}$, OWSP performs the best.

\section{Approximate Sampling Policies Under Finite Confounded Data}\label{procedure}
To deconfound according to NSP with finite confounded data is to deconfound the first $m$ confounded data.
For USP, we split the samples to the 4 groups as evenly as possible. That is,
we max out the bottleneck group/groups 
and distribute the excess data as evenly as possible among the remaining groups. 

For OWSP, we have $x_{yt} = \frac{ \hat{a}_{yt}}{\sum_y \hat{a}_{yt}}$,
and when implementing OWSP, we will first ensure
that the deconfounded samples are split 
as evenly as possible across treatment groups, 
and then within the each group, we split the samples close as possible to the outcome ratio.
\end{document}